\documentclass{article}
\usepackage{microtype}
\usepackage{graphicx}
\usepackage{subfigure}
\usepackage{booktabs} 
\usepackage{hyperref}

\usepackage{amsmath}
\usepackage{amsthm}
\usepackage{bm}
\usepackage{amssymb}
\usepackage{tikz}
\usetikzlibrary{positioning,automata}
\usepackage{xspace}

\usepackage{thm-restate}
\usepackage{multirow}
\usepackage[table]{xcolor}
\let\inf\undefined

\DeclareMathOperator*{\inf}{\vphantom{\sup}inf}

\DeclareBoldMathCommand{\I}{I}
\DeclareBoldMathCommand{\e}{e}
\DeclareBoldMathCommand{\f}{f}
\DeclareBoldMathCommand{\g}{g}
\DeclareBoldMathCommand{\a}{a}
\DeclareBoldMathCommand{\b}{b}
\DeclareBoldMathCommand{\d}{d}
\DeclareBoldMathCommand{\m}{m}
\DeclareBoldMathCommand{\p}{p}
\DeclareBoldMathCommand{\q}{q}
\DeclareBoldMathCommand{\v}{v}
\DeclareBoldMathCommand{\V}{V}
\DeclareBoldMathCommand{\x}{x}
\DeclareBoldMathCommand{\t}{t}
\DeclareBoldMathCommand{\X}{X}
\DeclareBoldMathCommand{\Y}{Y}
\DeclareBoldMathCommand{\z}{z}
\DeclareBoldMathCommand{\Z}{Z}
\DeclareBoldMathCommand{\M}{M}
\DeclareBoldMathCommand{\n}{n}
\DeclareBoldMathCommand{\ssigma}{\sigma}
\DeclareBoldMathCommand{\SSigma}{\Sigma}
\DeclareBoldMathCommand{\OOmega}{\Omega}
\DeclareBoldMathCommand{\y}{y}
\DeclareBoldMathCommand{\U}{U}
\DeclareBoldMathCommand{\w}{w}
\DeclareBoldMathCommand{\W}{W}
\DeclareBoldMathCommand{\L}{L}

\DeclareBoldMathCommand{\s}{s}
\DeclareBoldMathCommand{\S}{S}
\DeclareBoldMathCommand{\A}{A}
\DeclareBoldMathCommand{\B}{B}
\DeclareBoldMathCommand{\C}{C}
\DeclareBoldMathCommand{\D}{D}
\DeclareBoldMathCommand{\E}{\mathbb{E}}
\DeclareBoldMathCommand{\G}{G}
\DeclareBoldMathCommand{\H}{H}
\DeclareBoldMathCommand{\P}{\mathbb{P}}
\DeclareBoldMathCommand{\Q}{Q}
\DeclareBoldMathCommand{\R}{R}
\DeclareBoldMathCommand{\X}{X}
\DeclareBoldMathCommand{\mmu}{\mu}
\DeclareBoldMathCommand{\ones}{1}
\DeclareBoldMathCommand{\zeros}{0}

\newcommand{\cO}{\mathcal{O}}
\newcommand{\tcO}{\widetilde{\cO}}

\newcommand{\timeHorizon}{n}

\newcommand{\pullsNumber}{T}

\newcommand{\Exp}{\E}
\newcommand{\Pro}{\P}

\newcommand{\TODO}[1]{
\ifmmode
\text{\textcolor{red}{TODO: #1}}
\else
\textcolor{red}{TODO: #1}
\fi
}

\newcommand{\CommaBin}{\mathbin{\raisebox{0.5ex}{,}}}

\renewcommand{\epsilon}{\varepsilon}
\renewcommand{\hat}{\widehat}
\renewcommand{\tilde}{\widetilde}
\renewcommand{\bar}{\overline}

\newcommand{\Real}{\mathbb{R}}
\newcommand{\Integer}{\mathbb{N}}

\newcommand{\stateSpace}{X}
\newcommand{\actionSpace}{A}
\newcommand{\stateMDP}{x}
\newcommand{\action}{a}
\newcommand{\regret}{r}
\newcommand{\transFunc}{f}
\newcommand{\policy}{\pi}
\newcommand{\depth}{h}
\newcommand{\reward}{r}
\newcommand{\discount}{\gamma}
\newcommand{\valueF}{V}
\newcommand{\QvalueF}{Q}
\newcommand{\tree}{\mathcal{T}}

\newcommand{\Rmax}{R_{\max}}
\newcommand{\hmax}{h_{\rm max}}
\newcommand{\pmax}{p_{\rm max}}

\newcommand{\StroquOOL}{\normalfont \texttt{StroquOOL}\xspace}
\newcommand{\OLOP}{\normalfont \texttt{OLOP}\xspace}
\newcommand{\OPD}{\normalfont \texttt{OPD}\xspace}
\newcommand{\SOO}{\normalfont \texttt{SOO}\xspace}
\newcommand{\stopalgo}{\texttt{StOP}\xspace}
\newcommand{\metagrill}{\texttt{\textup{TrailBlazer}}\xspace}

\newcommand{\SequOOL}{\normalfont{\texttt{SequOOL}}\xspace}
\newcommand{\UCT}{\normalfont{\texttt{UCT}}\xspace}

\newcommand{\HOO}{\normalfont\texttt{HOO}\xspace}

\newcommand{\platypoos}{\normalfont{\texttt{\textcolor[rgb]{0.5,0.2,0}{PlaT$\gamma$POOS}}}\xspace}

\DeclareMathOperator*{\argmax}{arg\,max}

\newcommand{\depthOp}{\bot}
\newcommand{\lambertW}{W}

\newcommand{\branchF}{\kappa}

\newcommand{\StoSOO}{\texttt{StoSOO}\xspace}

\newcommand{\pa}[1]{\left(#1\right)}

\newtheorem{definition}{Definition}
\newtheorem{proposition}{Proposition}

\newtheorem{remark}{Remark}
\newcommand{\todo}[1]{}
\newcommand{\todov}[1]{}
\newcommand{\todovout}[1]{}
\newcommand{\todop}[1]{}
\newcommand{\todopout}[1]{}
\newcommand{\todomi}[1]{}
\newcommand{\todom}[1]{}
\newcommand{\todoji}[1]{}
\newcommand{\todoj}[1]{}

\usepackage[accepted]{icml2019}
\icmltitlerunning{Scale-free adaptive planning for deterministic dynamics \& discounted rewards}

\begin{document}
\twocolumn[
\icmltitle{Scale-free adaptive planning for deterministic dynamics \& discounted rewards}
\begin{icmlauthorlist}
\icmlauthor{Peter L.\nobreak\hspace{0.25em}Bartlett\!\nobreak\hspace{.06em}}{to}
\icmlauthor{Victor Gabillon\!\nobreak\hspace{.06em}}{ta}
\icmlauthor{Jennifer Healey\!\nobreak\hspace{.06em}}{ti}
\icmlauthor{Michal Valko\!\nobreak\hspace{.06em}}{tu}
\end{icmlauthorlist}
\icmlaffiliation{to}{\nobreak\hspace{.04em}University of California, Berkeley, USA}
\icmlaffiliation{ta}{\nobreak\hspace{.04em}Noah's Ark Lab, Huawei Technologies, London, UK}
\icmlaffiliation{ti}{\nobreak\hspace{.04em}Adobe Research, San Jose, USA}
\icmlaffiliation{tu}{\nobreak\hspace{.04em}SequeL team, INRIA Lille - Nord Europe, France}
\icmlcorrespondingauthor{Victor Gabillon}{victor.gabillon@huawei.com}
\icmlkeywords{Machine Learning, ICML}
\vskip 0.3in
]

\printAffiliationsAndNotice{}  

\begin{abstract}  
We address the problem of planning in an environment with deterministic dynamics and stochastic discounted rewards under a limited numerical budget where the ranges of both rewards and noise are unknown.  We introduce \platypoos, an adaptive, robust, and efficient alternative to the \OLOP (open-loop optimistic planning) algorithm.  Whereas \OLOP requires a priori knowledge of the ranges of both rewards and noise, \platypoos \textit{dynamically adapts} its behavior to both.  This allows \platypoos to be immune to two vulnerabilities of \OLOP: failure when given underestimated ranges of noise and rewards and inefficiency when these are overestimated. \platypoos additionally adapts to the global smoothness of the value function.  
\platypoos acts in a provably more efficient manner vs.\,\OLOP when \OLOP is given an overestimated reward and show that in the case of no noise, \platypoos learns exponentially faster.
\end{abstract}

\section{Introduction}
We consider the problem of planning in a general \textit{stochastic environment} with \textit{deterministic dynamics} and \textit{discounted rewards}. Our goal is to recommend the best first action for an agent to take from a given state.  We envision that the discount factor $\discount$ is known and that our learner has a limited allocation of $n$ interactions to spend querying a generative model of the environment.  The objective is to maximize the sum of discounted rewards of the best sequence of actions following from the recommended first action.  This is equivalent to minimizing the \textit{simple regret}.  We introduce the algorithm  \platypoos, \underline{Pla}nning wiTh $\discount$ Plus an Online Optimization Strategy, as a robust and efficient scale-free alternative to the \OLOP algorithm (open-loop optimistic planning, \citealp{bubeck2010open,leurent2019practical}) for  this setting.
Our algorithm implements a scale-free function optimization strategy similar to \SequOOL~\cite{bartlett2019simple} rather than an upper-confidence-bound approach which allows us to efficiently adapt to the problem space \textit{without  prior knowledge of the ranges of the noise or the rewards}.

Planning in a stochastic environment is an important setting often modeled by Markov decision processes (MDPs, \citealp{puterman1994markov,bertsekas1996neuro-dynamic}). One approach to solving these settings is to find the optimal policy that maximizes the expected sum of rewards and then generate an action recommendation according to that optimal policy.  Unfortunately, in most practical settings where we are limited by computational resources, finding this optimal policy is often not possible, especially when the state space becomes large.  Therefore, instead of trying to estimate the optimal policy of the MDP, we focus only on finding the \textit{best first action} given our budget.  We evaluate the performance of the recommendation in terms of the simple regret, the difference in reward between choosing the optimal first action vs.\,choosing our recommended first action and then in both cases choosing an optimal sequence of actions following the first action.  
This metric is often used to evaluate planning strategies that optimize numeric budgets~\citep{bubeck2010open,busoniu2012optimistic,grill2016blazing}, in contrast to the cumulative regret where we are penalized
during the search for querying sub-optimal actions.
Once the agent takes the first action and moves to the next state, our evaluation can be repeated with a new budget allocation and the following best first action can be recommended.  This allows us to approximately follow an optimal policy, action by action, in an online way. 
Previously, there have been several strategies proposed on how to efficiently allocate a numeric budget to search for an optimal value in a stochastic space.  Many of these have been successfully implemented using methods based on upper confidence bounds (UCBs) such as \UCT (Upper Confidence Trees, \citealp{kocsis2006bandit}). 
This approach has been proven to be very efficient in practice \citep{coulom2007efficient,gelly2006modifications,silver2016mastering}, however, \UCT can badly misbehave on some problems~\cite{coquelin2007bandit} and more theoretically sound approaches have been proposed~\citep{hren2008optimistic,bubeck2010open,busoniu2012optimistic,feldman2014simple,szorenyi2014optimistic,kaufmann2017monte,shah2019reinforcement}.  Some of these methods are connected to the ones from function optimization~\cite{bubeck2011x,munos2011optimistic,valko2013stochastic} as shown by~\citet{munos2014from}, however, one key difference is that in planning, as opposed to function optimization, the structure of the reward is a discounted reward, specifically a \emph{sum} of rewards \emph{discounted by factor} $\discount$.  This reward structure influences the behavior of the optimizers \citep{bubeck2010open}, in particular, the discount factor brings smoothness to the value function which in turn makes it easier to optimize. \platypoos exploits the effect of the discount factor to efficiently manage an adaptive planning strategy in the face of unknown ranges of noise and rewards. 


This adaptive strategy of \platypoos  makes it more robust and efficient in practice than other planning strategies.  For example, even though they are theoretically sound, the empirical performance of UCB-based approaches depends on the 
careful tuning of the upper confidence bound. If the upper confidence bound is too large then the UCB-based learner plays very conservatively by overestimating suboptimal options for many rounds.  Moreover, these UCBs might depend on instance parameters that are simply not known such as the \emph{range of the rewards} and the \emph{range of the noise}. We build on the function optimization approach of~\citet{bartlett2019simple} that does not use UCBs and obtains improved results over the state-of-the-art by adapting to the problem difficulty with a scale-free approach. \platypoos adapts this scale-free optimization to planning. This \textit{scale-free} property becomes a desired feature as machine learning gets closer to applications, whether it is online \cite{ross2013normalized,orabona2018scale} or deep learning \citep{orabona2017training}, since many parameters are never known.  

In terms of planning strategy, 
the \platypoos algorithm 
is an adaptive, robust, and efficient alternative to \OLOP. Whereas \OLOP \textit{requires the knowledge of where the ranges of both rewards and noise},
\platypoos dynamically adapts its behavior to the both ranges, as well as some potential additional global smoothness of the value function.  Our algorithm's ability to adapt allows it to avoid failure in cases where the ranges of noise and rewards are underestimated and to act more efficiently in cases where they are overestimated.  \platypoos  recovers the results of \OLOP while allowing  improvements in various classes of problems. 

\vspace{-0.1in}
\paragraph{Our contributions} 
We show that  \platypoos \vspace{-0.15in}
\begin{itemize}
\itemsep0em 
    \item adapts its behavior to an unknown range of rewards,\vspace{-0.05in}
    \item requires no apriori assumptions or knowledge on noise,\vspace{-0.05in}
    \item empirically learns much faster than UCB approaches,\vspace{-0.05in}
    \item gets the fast rate of deterministic planning in low noise for all regime; in particular, it learns exponentially faster than \OLOP when there happens to be no noise,\vspace{-0.05in}
   \item adapts also the  global smoothness $\rho$ and $\nu$ beyond the base smoothness provided by $\discount$.
\end{itemize}
We additionally address a realistic constraint where the agent can only \emph{reset} to the original state and not to any state it wishes. 
Our results hold for MDPs with deterministic dynamics and can equally be applied to open loop planning problems (as discussed by~\citealt{munos2014from}) where we search for the best sequence of actions, ignoring the actual states that are reached after each action \citep{bubeck2010open}.

Related algorithms, where the objective is to find the value of the state rather than to identify the best action, include \metagrill \cite{grill2016blazing}  and \stopalgo \cite{szorenyi2014optimistic}.   A key difference is that these algorithms are \textit{fixed confidence} and output a value using a small number of samples given an accuracy/probability, whereas our algorithm does exploration under a fixed budget of samples and guarantees \textit{how good} the found action is.  Even for simple multi-arm bandits, these two problems have different complexity~\cite{Carpentier16TB} and can only be equivalent under unrealistic side knowledge~\cite{gabillon2012best}.  These related algorithms are also impractical for our setting.  \metagrill uses confidence bounds that are humongous and \stopalgo takes exponential time.  Similar to \OLOP, both also need to know noise and reward ranges.

\section{Background}
\label{sec:Back}

We model our problem with an MDP with state space $\stateSpace$, action space $\actionSpace$ and dynamics such that taking the chosen action $\action_t$ at time $t$ deterministically transitions the system from $\stateMDP_t \in \stateSpace$ to state $\stateMDP_{t+1} \triangleq f (\stateMDP_t , \action_t )$ generating a reward $r_t \triangleq\reward(\stateMDP_t , \action_t )+\epsilon_t$, with $\epsilon_t$ being the noise.  We consider:
\begin{description}
	\vspace{-0.1in}
\item[deterministic rewards]
The evaluations are noiseless, that is for all $t$,
$\epsilon_t \triangleq 0\text{\ and\ } r_t\triangleq \reward(\stateMDP_t, \action_t).$
\item[stochastic rewards] The  evaluations  are
perturbed by a noise of range $b\in\Real_+$:
At any round, $\epsilon_t$ is a random variable, independent from noise at previous rounds,
\begin{equation} \label{eq:store}
\Exp \left[ \reward_t\,|\,\stateMDP_t\right] \triangleq \reward(\stateMDP_t,\action_t)\  \text{ and } \ |\reward_t-\reward(\stateMDP_t , \action_t )|\leq b.
\end{equation}
\end{description}
We assume that all rewards lie in the interval $[0 ,\Rmax]$ and while the state space may be large and possibly infinite, that the action space is finite, with $K$ available actions.  We treat an infinite time-horizon problem with discounted rewards where the discount factor ($0 \leq \discount < 1$) is \textit{known}.
For any possible policy $\policy$ : $\stateSpace \rightarrow \actionSpace,$ we define the value
function $\valueF^\policy:$  $ \stateSpace \rightarrow \Real$ associated to  $\policy$ as 
$\valueF^\policy(\stateMDP) \triangleq \sum
\discount^t \reward (\stateMDP_t , \policy(\stateMDP_t )),$
where $\stateMDP_t$ is the state of the system at time $t$ when starting from $\stateMDP$ (i.e.,
$\stateMDP_0 \triangleq \stateMDP$) and following policy $\policy$.
In the next definition, we also define the $Q$-value function $Q^\policy$ : $\stateSpace \times \actionSpace \rightarrow \Real$  associated to policy~$\policy$, for each state-action pair ($\stateMDP$, $\action$), as the value of playing action $a$ in
state $\stateMDP$ and the following $\policy$ thereafter.
\begin{definition} The $Q$-value function $\QvalueF^\policy$ of policy $\policy$ is
\[\QvalueF^\policy (\stateMDP, \action) \triangleq \reward(\stateMDP, \action) + \discount \valueF^\policy (f (\stateMDP, \action)).\]
\end{definition}
Notice that $\valueF^\policy (\stateMDP) = \QvalueF^\policy (\stateMDP, \policy(\stateMDP))$. We define the \emph{optimal} value function, and $Q$-value function respectively, as
$\valueF^\star (\stateMDP) \triangleq \sup_\policy \valueF^\policy  (\stateMDP)$ and $\QvalueF^\star (\stateMDP, \action) \triangleq \sup_{\policy} \QvalueF^\policy(\stateMDP, \action)$, which
corresponds to playing $\action$ first and optimally after. From the dynamic
programming, we have the Bellman equations
\citep{bertsekas1996neuro-dynamic,puterman1994markov},
\begin{align*}
\valueF^\star (\stateMDP) &= \max_{\action\in\actionSpace} \reward(\stateMDP, \action) + \discount \valueF^\star (f (x, a)),\\
\QvalueF^\star (\stateMDP, \action)  &= \reward(\stateMDP, \action) + \discount \max_{b\in\actionSpace} \QvalueF^\star (f (\stateMDP, \action), b).
\end{align*}
\vskip -.4cm
\noindent Let $[a:c]\triangleq\{a,a+1,\ldots,c\}$ with
$a,c\in\Integer$, $a\leq c$, and $[a]\triangleq[1:a].$ Let
$\log_{d}$ be the logarithm in base $d$, $d\in\Real$ and $\log$ without a subscript be the natural logarithm.\vspace{-.05cm}
%
\subsection{Optimistic planning under finite numerical budget}
We assume that we have a generative model of $\transFunc$ and~$\reward$ that generates simulated transitions and rewards. 
We want to make the best possible use of this model in order to recommend a best next action $a(n)$
such that the sum of the rewards resulting from playing $a(n)$ and then optimally afterwards is as close as possible to playing optimally from the beginning.
For that purpose, we define the performance loss
that we aim to minimize
$\regret_\timeHorizon$ as
\[
\regret_\timeHorizon \triangleq \max_{\action\in \actionSpace} \QvalueF^\star \pa{\stateMDP, \action} - \QvalueF^\star  \pa{\stateMDP, \action\pa{\timeHorizon}}.
\]
\subsection{The planning tree}
For a given initial state $\stateMDP$, consider the (infinite) planning tree defined
by all possible sequences of actions (thus all possible reachable states
starting from $\stateMDP$). Let $\actionSpace^\infty$ be the set of infinite sequences ($a_0 , a_1 , a_2 , \ldots$)
where $a_t \in \actionSpace$. The branching factor of this tree is the number of actions
$K \triangleq|\actionSpace|$. Since the dynamics are deterministic, to each finite sequence
$\action \in \actionSpace^d$ of length $d$ we assign a state that is reachable starting from
$\stateMDP$ by following this sequence of $d$ actions.
Using standard notation for alphabets, we write $\actionSpace^0 \triangleq \{\emptyset\}$ and $\actionSpace^\bullet$ for
the set of finite sequences. For $\action \in \actionSpace^\bullet,$ we let $\depth(\action)$ be the length of $\action$,
and $a\actionSpace^h \triangleq \{aa' , a' \in \actionSpace^h \}$, where $aa'$ denotes the sequence $a$ followed
by $a'$. We identify the set of finite sequences $\action \in \actionSpace^\bullet$ with the set of nodes
of the tree.
With $h'\leq h$ and $\action \in \actionSpace^h,$ we denote $\action_{[h']}$ the sequence of action composed of the $h'$ first actions from $\action$, i.e.,  $\{\action_0,\ldots,\action_{h'-1}\}$. We fix $\action_{[0]}\triangleq\emptyset$.
The value $v(\action)$ of an infinite sequence $\action \in \actionSpace^\infty$ is the discounted
sum of rewards along the trajectory starting from the initial state $\stateMDP$
and defined by the choice of this sequence of actions,
\[
v(\action) \triangleq
\sum_{t\geq 0}
\discount^t \reward(\stateMDP_t , \action_t ), \text{where } \stateMDP_0 = \stateMDP; \stateMDP_{t+1}  \triangleq \transFunc (\stateMDP_t , \action_t )\cdot
\]
\vskip -.3cm
%
\noindent Now, for any finite sequence $\action \in \actionSpace^\bullet$, or node, we define the value
$v(\action) \triangleq \sup_{\action' \in \actionSpace^\infty} v(\action\action')$. We write $v^\star \triangleq v(\emptyset) = \sup_{\action \in \actionSpace^\infty} v(a)$ for the optimal value at the initial state which is the root of the tree, $v^\star=\valueF^\star(\stateMDP)$.
We denote the set of optimal infinite sequence of action as $\actionSpace^\star$ which contains any $\action\in\actionSpace^\infty$ such that $v(a)=v^\star$.
We note the set of optimal finite sequence of actions of depth $h$ as $\actionSpace^{\star,h}$ which contains any $\action\in\actionSpace^h$ such that $v(a)=v^\star$.
 We also define the $u$-
and $b$-values for the lower- and upper- bounds on $v(\action)$ as
\[
u(\action) \triangleq
\sum^{\depth(\action)-1}_{t=0}
\discount^t \reward(\stateMDP_t , \action_t ), \text{and } b(\action) \triangleq u(\action) +
\frac{\discount^{\depth(\action)}\Rmax}{1-\discount}\cdot
\]
\noindent
Indeed, since all rewards are in $[0, \Rmax]$, we trivially have that $u(a) \leq
v(a) \leq b(a)$.
At any finite time $t$ an algorithm has opened a set of nodes,
which defines the expanded tree $\tree_t$ .
We say the learner
\emph{opens} (or expands)
a node $a$
with
$m$
evaluations if 
uses the generative model $f$ and $r$ to generate $m$ transitions and rewards
for the $K$ children nodes $aA$.  In the deterministic reward feedback,
$m
= 1$.  The bounds reported
in this paper are in terms of the total number of openings
$n$, instead of evaluations.  The
number of function evaluations is upper bounded by
$Kn$.
$T_{\stateMDP,\action}$
denotes
the total number of evaluations allocated to action $\action\in\actionSpace$ in state $\stateMDP$. 
We define, especially for the noisy case, the estimated
value of the reward $\hat \reward(\stateMDP,\action)$ of action $\action\in\actionSpace$ in state $\stateMDP$. Given the
$T_{\stateMDP,\action}$
evaluations
$\reward_1,\ldots,
\reward_{T_{\stateMDP,\action}}$, we let
$\hat{\reward}(\stateMDP,\action)\triangleq
\frac{1}{T_{\stateMDP,\action}}
\sum^{T_{\stateMDP,\action}}_{s=1}
\reward_s$ be
the empirical average of rewards obtained at when performing action action $\action\in\actionSpace$ in state~$\stateMDP$.
To ease notation, for $\action\in \actionSpace^m$ and $h\leq m$, we write 
$T_{\action} \triangleq \Exp[ T_{\stateMDP_{h(a)-1}, \action_{h(a)-1}}| \stateMDP_{t+1}\sim P(\cdot|\stateMDP_{t},\action_t), \stateMDP_0=\stateMDP]$
for the number of pulls to the last action in $a$. 
Similarly,  $\hat \reward_h(\action) \triangleq \Exp[ \hat \reward(\stateMDP_h,\action_h)| \stateMDP_{t+1}\sim P(\cdot|\stateMDP_{t},\action_t), \stateMDP_0=\stateMDP]$
and
$ \reward_h(\action) \triangleq \Exp[  \reward(\stateMDP_h,\action_h)| \stateMDP_{t+1}\sim P(\cdot|\stateMDP_{t},\action_t), \stateMDP_0=\stateMDP].$
In the case of deterministic dynamics, $\stateMDP_h$ is such that $\stateMDP_{t+1}\sim P(\cdot|\stateMDP_{t},\action_t)$ and $\stateMDP_0\triangleq\stateMDP{}$ is a fixed state from which we can  sample from if we have a full access to the generative model.
 Hence, for a finite sequence $\action \in \actionSpace^\bullet$ or a node, 
\[
\hat u(\action) \triangleq
\sum^{\depth(\action)-1}_{t=0}
\discount^t \hat \reward(\stateMDP_t , \action_t ) = \sum^{\depth(\action)-1}_{t=0}
\discount^t \hat \reward_t(  \action ).
\]
%
%
%
\noindent We assume the existence of at least one 
$a^\star 
\in\actionSpace^\infty$ 
for which
$\valueF^\star(x) = \sup_{a\in\actionSpace^\infty} v(a)$ and define a  smoothness for~$v$.
%
%

%
	\begin{proposition}\label{as:smoothu}
	There exists $\nu\in \left(0,\Rmax/\pa{1-\discount}\right]$ and $\rho\in (0,\discount]$ such that $\forall h \geq 0$, 
	$\forall \action \in \actionSpace^h,
	  u(\action)  
	\geq v(\action) - \nu\rho^h.$
\end{proposition}
\vspace{-.04cm}
Note that this holds automatically for $\nu=\Rmax/\pa{1-\discount}$ and $\rho=\discount$. For problems with an extra regularity this may also hold for some $\nu<\Rmax/\pa{1-\discount}$ and $\rho<\discount$. Note that our results automatically adapt to $\rho$ without knowing its value.
Moreover, note that while having a smoothness~$\rho$ means having rewards diminishing geometrically with depth with a ratio of~$\rho$, the constant $\nu$ is linked to the scale of variation of the $V$ which can often be realistically  smaller than  $\nu<\Rmax/\pa{1-\discount}$.
%
We now define a measure of the quantity of near-optimal sequences for the smoothness $\nu,\rho$.
\begin{definition}\label{def:neardim}
	For any $\nu > 0$ and $\rho \in (0,1)$, the \textbf{branching factor}
	$\branchF^v(\nu,\rho)$  with associated constant $C$,
	is defined as
	%
	%
	%
	%
	\[
	\branchF^v(\nu,\rho) \! \triangleq   \! \inf \!\left\{\!\branchF\geq 1\!:\! \exists C \!>\!    1, \forall h \geq 0, \mathcal N^v_h(3\nu\rho^h) \!\leq\!     C \branchF^{h} \!\right\}\!,
	\]
	where
	$\mathcal N^v_h(\epsilon)$ is the number of nodes
	$\action\in\actionSpace^h$    of depth $h$ such that
	$v(\action) \geq  v^\star   -  \epsilon$.
\end{definition}


We also define a related quantity but different  from 	$\branchF^v(\nu,\rho)$. In particular, we define 	$\branchF^u(\nu,\rho)$ that uses
	$\mathcal N^u_h(\epsilon)$ instead of $\mathcal N^v_h(\epsilon)$ where  $\mathcal N^u_h(\epsilon)$  is the number of nodes
$\action\in\actionSpace^h$    of depth $h$ such that
$u(\action) \geq  v^\star   -  \epsilon$.
Our results use  the new quantity $\branchF^u(\nu,\rho)$, recover the previous results using  $\branchF^v(\nu,\discount)$ in the method of~\citet{bubeck2010open} and go beyond them.
Indeed, theirs are only formulated for $\rho = \discount$ while we prove the following claim in Appendix~\ref{otbf}.
%
%
\begin{restatable}{proposition}{propo}{}
\label{pro:equiv}
$\branchF^u(\nu/2,\discount) \leq \branchF^v(\nu,\discount)\leq \branchF^u(2\nu,\discount).$ 	
\end{restatable}
%
%
\section{Optimization vs.\,planning}
%
%

Our  \platypoos approach is a very close sibling to the flat optimization algorithm, \StroquOOL~\cite{bartlett2019simple}, however, we explain how the structure of the planning setting and the discount factor $\gamma$ shaped our approach.  The optimal application of an optimization algorithm to the planning setting is not straightforward as discussed by~\citet{bubeck2010open}.  In their Section 2.2, \citet{bubeck2010open} show that optimization can be applied to the planning problem either in a na\"ive way or a \textit{good} way.  The authors take as an example the uniform planning problem.  The na\"ive and good strategies are evaluated by comparing the uncertainty $|u(a)-\hat u(a)|$ of their estimates $\hat u(a)$.  Both strategies collect rewards identically, evaluating $u(\action)$ for all the $K^H$ nodes $a\in\actionSpace^H$ at a fixed depth $H\triangleq h(a)$ by allocating one episode (of length $H$) for each $a$ and receiving $\reward_{h,a}$, $1\leq h\leq H$.  In the na\"ive  version, for all sequences $a$, the estimation of $u(a)$ uses only the~$H$ samples collected in the one episode related to $a$.  Here $\hat u(a) \triangleq \sum^{\depth(\action)-1}_{h=0} \discount^t \reward_{h,a}$ and $T_{a_{[h]}}=1$.  In contrast, the good planning strategy \textit{reuses} estimates.  For two distinct sequences $a$ and $a'$, the good strategy reuses any sample of $r_{h,a}$ it gets for the estimation of both $u(a) $ and $u(a') $ if  $a_{[h]}=a'_{[h]}$.  In this case, $\hat u(a) \triangleq \sum^{\depth(\action)-1}_{h=0} \discount^t \frac{1}{K^{H-h}} \sum_{a':a_{[h]}=a'_{[h]}} \reward_{h,a'}$ and $T_{a_{[h]}}=K^{H-h}$.  This comparison is used to demonstrate the advantage of the use of the cross-sequence information to concentrate the estimate of the mean reward associated with each action more efficiently.  It is by using this type of cross-sequence information that \OLOP is able to obtain a reduced regret over a na\"ive application of \HOO for optimization~\cite{bubeck2011x} to the planning problem.

The previous discussion on cross-sequence information is tied to the case of uniform exploration strategies. 
 Good uniform strategies guarantee
 $T_{a_{[h]}}=K^{H_u-h}$ by exploring until a reasonably shallow depth $H_u$ but sampling all $K^{H_u}$ nodes and sharing cross sequence information.
 In the case of \OLOP, \HOO, or, particularly \StroquOOL, only a subset~$S$ of size $|S|\ll K^{H_u}$ of the most promising nodes are explored but at a deeper depth $H_s\gg H_u$. Therefore, obtaining a lower bound on the number of sequence of actions at depth $ H_s$ that contains $a$ for a given sub-sequence of actions $a$ at depth $h< H_s$ is complex in general. 
 Actually one can design a problem with two actions that would drastically limit the amount of cross-sequence information for the  optimal sequence of actions $a^\star$ in \StroquOOL. 
 As a result, applying \StroquOOL with information sharing may not be enough and we chose to algorithmically ensure  hat a node at depth $h+1$ will be pulled $\discount^2$ times less than  a node at depth $h$.
 Indeed, using Chernoff-Hoeffding inequalities, we have $|u(a)-\hat u(a)|\leq \sum_{h=0}^{h(a)-1} \discount^h / \sqrt{T_{a_{[h]}}}$. 
However, \platypoos, through some simple reparametrization, remains very close to \StroquOOL and we leave as an open question whether equivalent theoretical guarantees could be proved directly for \StroquOOL applied to planning.

\section{Deterministic dynamics and rewards}
\label{sec:deterdeter}

In this section, we consider a simpler case of deterministic rewards in order to introduce our new ideas. 
The  evaluations are  noiseless, that is $\forall t$, $\epsilon_t \triangleq 0$ and $r_t\triangleq\reward(\stateMDP_t , \action_t )$.

\begin{figure}
	\centering
	\framebox{
		~    \begin{minipage}{.45\textwidth}
			\textbf{Input:} 
			$\timeHorizon$, $\actionSpace$\\[.05cm] 
			\textbf{Initialization:}
			open $t_{0,1};$  $\hmax \gets \left\lfloor\timeHorizon / \bar\log(\timeHorizon)\right\rfloor\!$ \\[.1cm]
			\textbf{For} $h=1$ to  $\hmax$	\vspace{.05cm}
			
			$\quad$ open $\left\lfloor \hmax / {h}\right\rfloor $ nodes $a_{h,i}$ of depth $h$ \\ \vspace{.05cm}$\quad$ with largest values  $ u(a_{h,i})$
			
			\textbf{Output} $ x(\timeHorizon)\gets\argmax\limits_{a_{h,i}: \in \tree}u(a_{h,i})$
		\end{minipage}    
	}
	\caption{Algorithm for free planning with no reset condition}\label{fig1}
\end{figure}

In Figure~\ref{fig1}, we provide the \SequOOL~\cite{bartlett2019simple} algorithm  applied to planning.
In this case, it is straightforward to follow the analysis of \SequOOL in order to obtain the same rates of simple regret
 as the state of the art algorithm \OPD for the doubly deterministic case~\cite{hren2008optimistic,munos2014from}, up to logarithmic factors; 
 and get the result\footnote{$\bar\log(n)$ is the $n$-th harmonic number} of Theorem~\ref{th:sequool}.
 This direct usage was already  discussed by \citet[Section 5.1]{munos2014from}.
Using \SequOOL for planning already permits to have an algorithm that does not use the parameter $\Rmax$ and that adapts to extra smoothness in the value function $\nu$, $\rho$ as discussed Section~\ref{sec:Back}. Note that to obtain similar adaptations to $\Rmax$, $\nu,$ and $\rho,$ we could have already used \SOO~\cite{munos2014from}. However, \SOO \textit{does not} come with optimal simple regret~\cite{bartlett2019simple}.

\clearpage



	\begin{restatable}{theorem}{restathsequool}\label{th:sequool}
	For any planning problem with associated $(\nu,\rho)$, and branching factor $\branchF \triangleq\branchF^u(\nu,\rho)$, 
   the simple regret of {\SequOOL}  is 	after $\timeHorizon$ rounds bounded as follows.
	\begin{enumerate}
		\vspace{-0.1in}
		\itemsep0em
	    \item[\textbullet] If $\branchF=1$, then$\: r_\timeHorizon\leq\	\nu\rho^{\frac{1}{C}\left\lfloor\frac{\timeHorizon}{\bar\log\,\timeHorizon}\right\rfloor}.$
	
	    \item[\textbullet] If $\branchF>1$, then $\: r_\timeHorizon = \cO\left( \nu\left(\frac{\log{\branchF}}{C}\left\lfloor\frac{\timeHorizon}{\bar\log\,\timeHorizon}\right\rfloor \right)^{-\frac{\log{1/\rho}}{\log{\branchF}}}
	    \right)\!.$
	    \vspace{-0.1in}
	\end{enumerate}
	
	%
	%
\end{restatable}

\paragraph{On the reset condition}
In practice, physical constraints can force the exploration to interact with the unknown environment under a reset condition. This means that there exists a starting state $\stateMDP$ and  a trajectory needs to start from this state and following a given policy.  Therefore, if we wish to collect a sample from an arbitrary state~$\stateMDP'$ after a reset, we must first reach that state from the starting state $\stateMDP$. 

\SequOOL is a strategy that explores the MDP deeper and deeper from an initial starting state $\stateMDP$.
The original \SequOOL strategy did not consider any reset condition, so to make a comparable analysis  
we will say that to `open' a node $a$ you first need to reach $a$ at a budget cost of 
$h(a)$ which is equal to the depth of $a$.  This additional cost has the consequence that \SequOOL with reset will not be able to explore as deeply as without.
A~na\"ive extension is shown in Figure~\ref{fig2}.  Under a total limited budget of $n$, the number of nodes now open at depth $h$ is $\cO(n/h^2)$ instead of $\cO(n/h)$
and the maximal depth is now of order of $\sqrt{n}$ instead of $n$. 
However, this does not influence the simple regret when $\kappa>1$ by more than numerical constants as shown in Theorem~\ref{th:sequool2}.  When $\kappa>1$, the exponentially diminishing simple regret remains but is changed from $\rho^{n}$ to  $\rho^{\sqrt{n}}$.
\begin{figure}
\vspace{-0.05in}
	\centering
	\framebox{
		~    \begin{minipage}{.45\textwidth}
			
			\textbf{Input:} 
			$\timeHorizon$, $\actionSpace$\\[.05cm] 
			\textbf{Initialization:}
			open $t_{0,1};$         $\hmax \gets \left\lfloor \frac{\timeHorizon} {\bar\log\,\timeHorizon}\right\rfloor$ \\[.1cm]
			\textbf{For} $h=1$\ to\ $\hmax$
			
			\vspace{.05cm}
			
			$\quad$ open  $\left\lfloor \hmax / {h^2}\right\rfloor $ nodes $a_{h,i}$ of depth $h$ \\ \vspace{.05cm}$\quad$ with largest values  $ u(a_{h,i})$
			
			\textbf{Output} $ x(\timeHorizon)\gets \argmax\limits_{a_{h,i}: \in \tree}u(a_{h,i})$
		\end{minipage}    
	}
	\caption{Algorithm for constraint planning (restart)}\label{fig2}
\end{figure}


	\begin{restatable}{theorem}{restathsequool}\label{th:sequool2}
	For a planning problem with associated $(\nu,\rho)$, and branching factor $\branchF\triangleq\branchF^u(\nu,\rho)$, 
	after $\timeHorizon$ rounds,    the simple regret of {\SequOOL}  with reset condition verifies:
	\vspace{-0.2in}
	\begin{enumerate}
		\itemsep0em
	    \item[\textbullet] If $\branchF=1$, then$\: r_\timeHorizon\leq\	\nu\rho^{\sqrt{\frac{1}{C} \left\lfloor\frac{\timeHorizon}{\bar\log\,\timeHorizon}\right\rfloor}}.$
	
	    \item[\textbullet] If $\branchF>1$, then $\: r_\timeHorizon\leq  \cO\left( \nu\left(\frac{\log{\branchF}}{C}\left\lfloor\frac{\timeHorizon}{\bar\log\,\timeHorizon}\right\rfloor \right)^{-\frac{\log{1/\rho}}{\log{\branchF}}}
	    \right)$.
	\end{enumerate}
	%
\end{restatable}


\section{Deterministic dynamics, stochastic rewards}

\begin{figure}
\vspace{-0.05in}
	\centering
	\framebox{
		~    \begin{minipage}{.45\textwidth}
			
			\textbf{Input:} 
	$\timeHorizon$, $\actionSpace$\\[.05cm]
			\textbf{Initialization:}
	open the root node $\emptyset$, $ \hmax$ times\\            $\hmax \gets \left\lfloor\frac{\timeHorizon}{2(\log_2\timeHorizon+1)^2}\right\rfloor\!\!\CommaBin \pmax \gets \left\lfloor\log_2\left( \hmax\right)\right\rfloor$ \\[.45cm]
		\textbf{For} $h=1$ to $\hmax$\vspace{.05cm}  \textit{\textbf{\textcolor{gray}{$\hfill\blacktriangleleft$  exploration $\blacktriangleright$}}}
		
		~~\textbf{For} $p=\left\lfloor\log_2(\hmax/\left\lceil h^2\discount^{2h}\right\rceil)\right\rfloor$ down to $0$
		\\
		\phantom{aaa}open  $\left\lceil h 2^p\discount^{2h}\right\rceil$ times the at most $\left\lfloor \frac{\hmax}{h \left\lceil h 2^p\discount^{2h}\right\rceil}\right\rfloor $  \\       \phantom{aaa}non-opened nodes $\action^{h,i}\in \actionSpace^{h}$ 
		with highest values 
		\\  \phantom{aaa}$\hat u(\action^{h,i})$ and given   $\pullsNumber_{\action^{h,i}}\geq \left\lceil (h-1) 2^p\discount^{2(h-1)}\right\rceil$\\[.45cm]
		\textbf{For} $p\in[0:\pmax]$ \textit{\textbf{\textcolor{gray}{$\hfill\blacktriangleleft$  cross-validation $\blacktriangleright$}}} \vspace{.05cm} \\ 
		\phantom{aaa} \textbf{evaluate} $(t+1)\discount^{2t}\hmax(1-\discount^2)^2$ times \\ \phantom{aaa} the
		actions at round $t$, $a_t^p$, of the \textit{candidates:}\\
		\phantom{aai} $a^p \gets \argmax\limits_{\action\in \actionSpace^\bullet:\forall t\in[2:\depth(\action)],\pullsNumber_{\action_{[t]}}\geq \left\lceil (t-1) 2^p\discount^{2(t-1)}\right\rceil} \hat u(\action)$ 
		
		
		\textbf{Output} $a^\timeHorizon \gets \argmax\limits_{\{a^p ,p\in[0:\pmax]\}} \hat u(a^p)$

		\end{minipage}    
	}
	\caption{The \platypoos  algorithm}\label{fig:plati}
\end{figure}



 We now describe the \platypoos algorithm detailed in Figure~\ref{fig:plati}.
In the presence of noise, it is natural to evaluate the cells multiple times, not just one time as in the deterministic case.
The amount of times a cell should be evaluated to differentiate its value from the optimal value of the function depends on the gap between these two values as well as the range of noise. As we do not want to make \emph{any} assumptions on knowing these quantities, our algorithm tries to be robust to any potential values by not making a fixed choice on the number of evaluations.  Intuitively, we do this following a path similar to \StroquOOL~\cite{bartlett2019simple} by using a modified version of   
\SequOOL, denoted \SequOOL{}$(m),$ that allows us to evaluate cells $m$ times, whereas for \SequOOL, $m=1$.  Evaluating cells more times ($m$ large) leads to a better quality of the mean estimates in each cell, however, as a trade-off, it uses more evaluations per depth.  This would normally limit us from exploring deep depths of the partition, 
however, \platypoos takes advantage of the knowledge of~$\discount$ which gives less weight to reward collected deeper in the tree. In order to obtain the concentration results for a node $a$ on $\hat u(a)- u(a)$ in Lemma~\ref{l:event}, \platypoos uses a Chernoff-Hoeffding result that gives with high probability, $\hat u(a)- u(a) \leq \sqrt{\sum_{h=0}^{h(a)-1} \discount^{2h}/T_{a_{[h]}}}$ and balances the range of confidence intervals at different depths.  Therefore, \platypoos tends to pull less with deeper depth as the number of pulls for a fixed $m$ is $\left\lceil hm\discount^{2h} \right\rceil$ where the additional $h$ factor ensures that the sum of confidence interval until depth $h$ is bounded for  any $h$.
\platypoos then \emph{implicitly} performs $\log n$ instances of \SequOOL{}$(m)$ each with a number of evaluations of $m=2^p$, where $p\in[0:\log n]$. 
In Figure~\ref{fig:plati},
remember that `opening' a node means `evaluating' its children actions.
The algorithm opens nodes by sequentially diving them deeper and deeper from the root node $h=0$ to a maximal depth of $\hmax$.
At depth~$h$, we allocate, in an almost decreasing fashion, different number of evaluations $\left\lceil h2^p\discount^{2h}\right\rceil$ to the  nodes with highest value of that depth, with $p$ starting at $\left\lfloor\log_2(\hmax/h)\right\rfloor$ down to $0$. The best node that has been evaluated at least $\cO(\hmax/h)$ times is opened with $\cO(\hmax/h)$ evaluations, the two next best cells that have been evaluated at least $\cO(\hmax/(2h)$ times are opened with $\cO(\hmax/(2h))$ evaluations, the four next best cells that have been evaluated at least $\cO(\hmax/(4h))$ times are opened with $\cO(\hmax/(4h))$ evaluations and so on, until some $\cO(\hmax/h)$ next best cells that have been evaluated at least once are opened with one evaluation.  
More precisely,   given, $p$ and $h$, we open, with $\left\lceil h 2^p \discount^{2h} \right\rceil$ evaluations, the $\left\lfloor \hmax/(h\left\lceil h 2^p \discount^{2h} \right\rceil)\right\rfloor$ non-previously-opened nodes $a^{h,i}\in\actionSpace^h$ with highest values  $\hat u(a^{h,i})$ and given that  $\pullsNumber_{\action^{h,i}}\geq \left\lceil (h-1) 2^p\discount^{2(h-1)}\right\rceil$. 
The maximum number of evaluations of any node 
is $2^{\pmax}$, with $2^{\pmax}=\cO(\hmax)$ as $\pmax \triangleq \left\lfloor\log_2\left( \hmax\right)\right\rfloor$. 
For each $p\in[0:\pmax]$, the candidate output  $a^p$ is the node~$a$ with the highest estimated value such that all actions leading to that node have been evaluated in the following way $\forall t\in[2:\depth(\action)],\pullsNumber_{\action_{[t]}}\geq \left\lceil (t-1) 2^p\discount^{2(t-1)}\right\rceil$.
We set $\hmax\triangleq\left\lfloor \timeHorizon/(2(\log_2\timeHorizon+1)^2)\right\rfloor\!.$ 
%
\subsection{Analysis of \platypoos}
%
%
$\depthOp_{h,p}$ is
the depth of the deepest opened node, $\action$ with at least $\left\lceil h2^p\discount ^h\right\rceil$ evaluations such that there is a~$a^\star\in \actionSpace^\star$ with $a^\star\triangleq ab$, with $b\in\actionSpace^\infty$, 
at the end of the opening of depth~$h$.
\begin{restatable}{lemma}{restahstarStoTwo}\label{lem:hstarSto2}
	For any planning problem with associated $(\nu,\rho)$ as in Property~\ref{as:smoothu}), on event $\xi$ defined in Appendix~\ref{lemproofs}, for any depth $h \in \left[\hmax\right],$ for any $p\in[0:\left\lfloor\log_2(\hmax/(h^2\discount^{2h}))\right\rfloor]$, we have 	$\depthOp_{h,p}= h$ if (1) and (2) simultaneously hold:\\
	(1)	 $b \sqrt{\log(4n/\delta) / 2^{p+1}} \leq \nu\rho^h$
	\\
	(2)	
	 We distinguish cases and express the condition in each:\vspace{-.2cm}
	 \begin{align*}  \hspace{-.5cm}
     &\textbf{Case 1)
	\quad $ h2^p\discount ^{2h}\leq 1$}:\\
 &  \frac{\hmax}{h}=  \frac{\hmax}{h\left\lceil h 2^{p}\discount^{2h}\right\rceil}\geq  C\branchF(\nu,\rho)^{h}\\
  &  \text{and for all } h'\in[h],  \frac{\hmax}{h'^2 2^{p+1}\discount^{2h'}}\geq  C\branchF(\nu,\rho)^{h'}
	 \\ \
	&\textbf{Case 2) 
$ h2^p\discount ^{2h}\geq 1$}:\\ \vspace{-.1cm}
&\quad \textbf{Case 2.1) $\discount^{2}\branchF^u\geq 1$}: \frac{\hmax}{h^2 2^{p+1}\discount^{2h}}\geq  C\branchF(\nu,\rho)^{h}\\
&\quad\textbf{Case 2.2) $\discount^{2}\branchF^u\leq 1$}:\frac{\hmax}{h^2 2^{p+1}}\geq  C
	 \end{align*}
\end{restatable}  \vspace{-.2cm}
\noindent
Lemma~\ref{lem:hstarSto2} gives two conditions so that the cell containing a $a^\star \in\actionSpace^\star$ is opened at depth $h$. This holds if (1) \platypoos opens, with $\left\lceil h2^p \discount^{2h}\right\rceil$ evaluations, more cells at depth $h$ than the number of near-optimal cells at depth $h$ ($\hmax / \pa{h^2 2^p \discount^{2h}} \geq  C\branchF(\nu,\rho)^{h}$ if $\discount^{2}\branchF^u\geq 1$ and 
$\hmax / {h2^p} \geq  C$ if $\discount^{2}\branchF^u\leq 1$)
 and (2) the $\left\lceil h2^p \discount^{2h}\right\rceil$ evaluations are sufficient to discriminate the empirical average of near-optimal cells from the empirical average of sub-optimal cells ($b \sqrt{\log(4n/\delta)/2^{p+1}} \leq \nu\rho^h$). 
%
%
To state the next theorems, we introduce $\tilde h$, $\tilde h_1,$  and $\tilde h_2$ three positive real numbers satisfying respectively the equations:
\begin{enumerate}
    \item[]$\hmax\nu^2\rho^{2\tilde h_1} /\pa{\tilde h_1b^2g_{n,b}^{\delta,\Rmax}\discount^{2\tilde h_1}}=C\branchF^{\tilde h_1} $ and
    \item[]${\hmax\nu^2\rho^{2\tilde h_2}} / \pa{\tilde h_2b^2g_{n,b}^{\delta,\Rmax}}=C,$ where \\ \vspace{-.2cm}
    \item[] $\tilde h_1 =    \frac{1}{\log(\discount^2\branchF/\rho^2)}\log\left(\frac{\bar n_1}{\log \bar n_1 } \right)+o(1)$ and
    \item[]$\tilde h_2 =    \frac{1}{\log(1/\rho^2)}\log\left(\frac{\bar n_2}{\log \bar n_2 } \right)+o(1)$ with
     \item[]$\bar n_1 \triangleq \frac{\nu^2\hmax \log(\discount^2\branchF/\rho^2)}{Cb^2g_{n,b}^{\delta,\Rmax}}\CommaBin$ $\bar n_2 \triangleq \frac{\nu^2\hmax \log(1/\rho^2)}{Cb^2g_{n,b}^{\delta,\Rmax}}\CommaBin$
\end{enumerate}
where $g_{n,b}^{\delta,\Rmax}\triangleq\pmax\log(\Rmax n^{3/2}/b(1-\delta))$. $\tilde h$ is defined similarly in Equation~\ref{eq:brokenbells} in Appendix~\ref{app:proofALLmastro}. The quantities $\tilde h_1$ and  $\tilde h_2$ give the respective depths of deepest cell opened by \platypoos that contains a $a^{\star}$ with high probability in the cases $\discount^2\branchF \geq 1$ and  
$\discount^2\branchF \leq 1$. Additionally,  $\tilde h_1$ and $\tilde h_2$ also let us characterize for which regime of the noise range~$b$ we recover results similar to the loss of the deterministic case.
Discriminating on the noise regimes, we now state two of our results, Theorem~\ref{th:highnoise} for a high noise and 
Theorem~\ref{th:lownoise} for a low one. A more exhaustive list of results is in the Appendix~\ref{app:proofALLmastro} or in the Table~\ref{tab:1}.
%

\begin{table*}[t]
	\centering
	\bgroup
\def\arraystretch{2}
	\begin{tabular}{|c|c|c|c|c|}
		\hline
&		\multicolumn{2}{|c|}{$\discount^2\kappa \leq 1$} &  \multicolumn{2}{|c|}{$\discount^2\kappa \geq 1$} \\ \hline
&		 \textit{High noise} (ii) & \textit{Low noise} (ii)& \textit{High noise} (iii) & \textit{Low noise} (iii)\\ \hline
\textit{High noise}\textit{} (i) 
& \cellcolor{gray!20}$m_{\nu,\discount}\left(\frac{n}{b^2}\right)^{-\frac{1}{2} }$ 
& $\nu\rho^{\sqrt{n}}$
& 	\cellcolor{gray!20} $m_{\nu,\discount}\left(\frac{n}{b^2}\right)^{-\frac{\log(1/\rho)}{\log(\discount^2\kappa/\rho^2)} }$ 
&  	$\nu\left(\frac{n}{b^2}\right)^{-\frac{\log(1/\rho)}{\log(\kappa)} }$\\ \hline
	\multirow{2}{*}[-.5em]{\textit{Low noise }(i) }&		\multirow{2}{*}{$m_{\nu,\discount}\left(\frac{n}{b^2}\right)^{-\frac{\log(1/\rho)}{\log(\kappa)} }$}
	&  $\kappa =1: \nu\rho^{n}$
	& \multirow{2}{*}{$m_{\nu,\discount}\left(\frac{n}{b^2}\right)^{-\frac{\log\pa{1/\rho}}{\log(\kappa)} }$ }
	& \multirow{2}{*}{$\nu\left(\frac{n}{b^2}\right)^{-\frac{\log\pa{1/\rho}}{\log(\kappa)} }$} \\ \cline{3-3}
& &  $\kappa >1: \nu\left(\frac{n}{b^2}\right)^{-\frac{\log(1/\rho)}{\log(\kappa)} }$ & & \\ \hline
	\end{tabular}
	\egroup
	\vspace{0.01in}
	\caption{Rates of the our upper bounds on the simple regret of \platypoos for various classes.  The condition on noise (i) is whether the noise $b$ verifies	${\nu^2\rho^{2\tilde h}} / (\discount^{2\tilde h}\tilde h b^2g_{n,b}^{\delta,\Rmax}) \leq 1$
	 The condition on noise (ii) is whether	${\nu^2\rho^{2\tilde h_2}} / (b^2g_{n,b}^{\delta,\Rmax}) \leq 1.$	
	 The condition on noise (iii) is whether	${\nu^2\rho^{2\tilde h_1}} /  (b^2g_{n,b}^{\delta,\Rmax}) \leq 1$. Moreover, $m_{\nu,\discount} \triangleq \max(1 / {1-\gamma^2},\nu)$.	 }
	\label{tab:1}
\end{table*}


\begin{restatable}{theorem}{restathhighnoise}{\textcolor{gray}{\textit{\textbf{High-noise regime}}}}\label{th:highnoise}
If the noise $b$ is high enough to verify both high-noise conditions as defined in the caption of Table~\ref{tab:1}, then
	after $\timeHorizon$ rounds, for any problem with associated $(\nu,\rho)$, and branching factor   $\branchF\triangleq\branchF^u(\nu,\rho)$, the simple regret of \platypoos obeys
	%
	%
\begin{equation*}
\Exp r_\timeHorizon~=~
\begin{cases} 
\tcO\left(\left(\frac{n}{b^2}\right)^{-\frac{1}{2} }\right) & if~ \discount^2\kappa\leq 1,\\
\tcO\left(\left(\frac{n}{b^2}\right)^{-\frac{\log(1/\rho)}{\log(\discount^2\kappa/\rho^2)} }\right)  & if~ \discount^2\kappa>1.
	\end{cases}
\end{equation*}
\end{restatable}
The proofs are in appendix~\ref{app:proofALLmastro}. They  are quite technical but they are simply based on checking the conditions of Lemma~\ref{lem:hstarSto2} under different  $b,\rho,\discount,\kappa$ regimes.
\begin{restatable}{theorem}{restathlownoise}{\textcolor{gray}{\textit{\textbf{Low-noise regime}}}}\label{th:lownoise}
 If the noise $b$ is low enough to verify both high-noise conditions as defined in the caption of
	 Table~\ref{tab:1}, then
	after $\timeHorizon$ rounds, for any problem with associated $(\nu,\rho)$, and branching factor  $\branchF\triangleq\branchF^u(\nu,\rho)$,   the simple regret of \platypoos obeys 
	%
	%
%
	\begin{equation*}
\Exp r_\timeHorizon~=~
\begin{cases} 
\tcO\left(\nu\rho^{n}\right) & if~ \kappa=1,\\
\tcO\left(\nu\left(\frac{n}{b^2}\right)^{-\frac{\log(1/\rho)}{\log(\kappa)} }\right)& if~ \kappa>1.
	\end{cases}
\end{equation*}
%
	%
\end{restatable}
%
%
%
\paragraph{Worst-case comparison with \OLOP \textcolor{gray}{\textbf{when $b$ is large and known}} }
In Table~\ref{tab:1}, we give our results for various classes of problems depending on the whether $\discount^2\kappa \geq 1$, whether $\kappa =1$ or $\kappa >1$, and several conditions for the range of the noise $b$. 
 The results for \OLOP were  distinguishing the results based  on $\discount^2\kappa$ being greater or smaller than $1$. For these two cases, we recover the same rate in term of $n$ for the simple regret, as displayed in Table~\ref{tab:1} with a grey background color, for instance taking $b=1$ like in \OLOP and having $\rho=\discount.$
However, we provide more specific treatment for sub-cases with associated improvements that we list and detail next.

\paragraph{Adaptation to the range of the noise $b$ without a prior knowledge}
Our analysis shows that \platypoos adapts favorably to the unknown range of noise. Already, in the standard cases discussed above, where the noise is large, our bound already adapts to the amount of noise as it scales with $n/b^2$. \OLOP requires an estimate  $\tilde b$ of $b$ and has a regret scaling with $n/\tilde b^2$ which is problematic in case of a wrong estimate  $\tilde b\gg b$. Moreover, we give technical conditions on the range of noise that shows when \platypoos gets improved rates.
When $\discount^2\kappa\geq 1$, \OLOP was already obtaining rates that were the same as the rates of deterministic reward case. Therefore, beyond the  $n/b^2$ improvement and the adaptation to extra smoothness that will be discussed latter, no more rate improvement should be expected. 
In the case $\discount^2\kappa\leq 1$, the improvement are even more striking. When the noise is very low, then contrary to the \OLOP rate of $\cO(1/\sqrt{n})$, we obtain the \textit{deterministic} rate of OPD~\cite{hren2008optimistic,munos2014from} which is either $n^{-\log(1/\rho)/\log\kappa}$ or $ \rho^n$.
These improved rates \textit{could not}  be obtained by \OLOP. Indeed,  \OLOP relies on 
 upper confidence bound (UCB) that uses a range of the noise $\tilde b$ as input,
 \[
 \bar u(a) =  \cO\left(\hat u(a) + \sum_{h=0}^{h(a)-1} \discount^h \tilde b \sqrt{\frac{1}{T_{a_{[a]}}}} + \frac{\Rmax\discount^{h(a)}}{1-\discount}\right)\!\cdot
 \]
This works if $\tilde b = b$. However, the true $b$ is unknown. If $b=0$, using any $\tilde b> 0$ will not result in an improved rate.

\paragraph{Adaptation to additional smoothness $\nu$ and $\rho$ beyond~$\discount$} 
As defined in Section~\ref{sec:Back}, we aim to adapt to the true smoothness $\nu,\rho$ of $V$ which can go beyond $ \discount$. We show that \platypoos is able to take advantage of  $\nu,\rho$ in a large portion of  cases. 
In most cases, the rate $\discount$ in \OLOP is replaced by $ \rho$ in \platypoos.
In the case where $\discount^2\kappa\geq 1$ we have \[r^{\OLOP}=\cO(n^{-\frac{\log(1/\discount)}{\log(\kappa)}}) \leq \cO(n^{-\frac{\log(1/\rho)}{\log(\discount^2\kappa/\rho^2)}})=r^{\platypoos}.\]
%
\begin{figure*}[t]
	\center
	\includegraphics[width=0.325\textwidth]{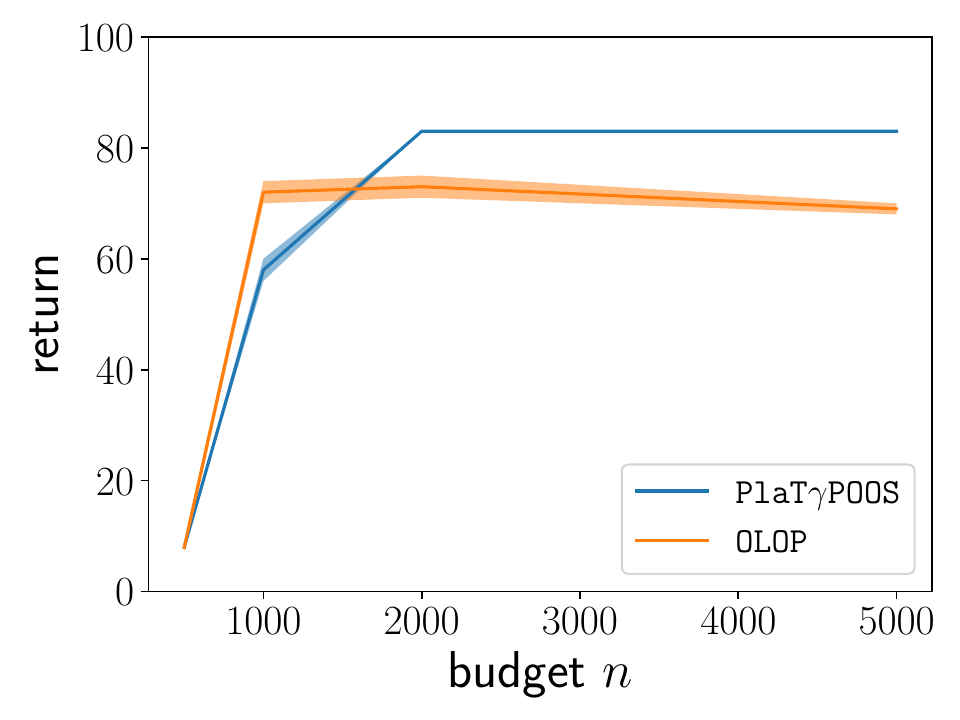}
	\includegraphics[width=0.325\textwidth]{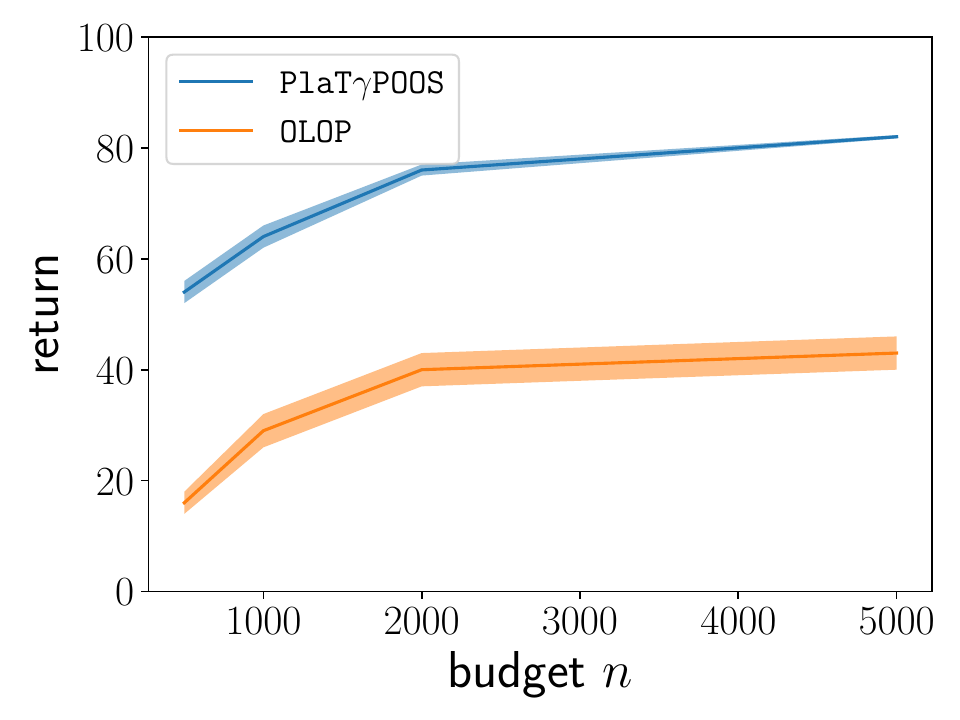}
		\includegraphics[width=0.325\textwidth]{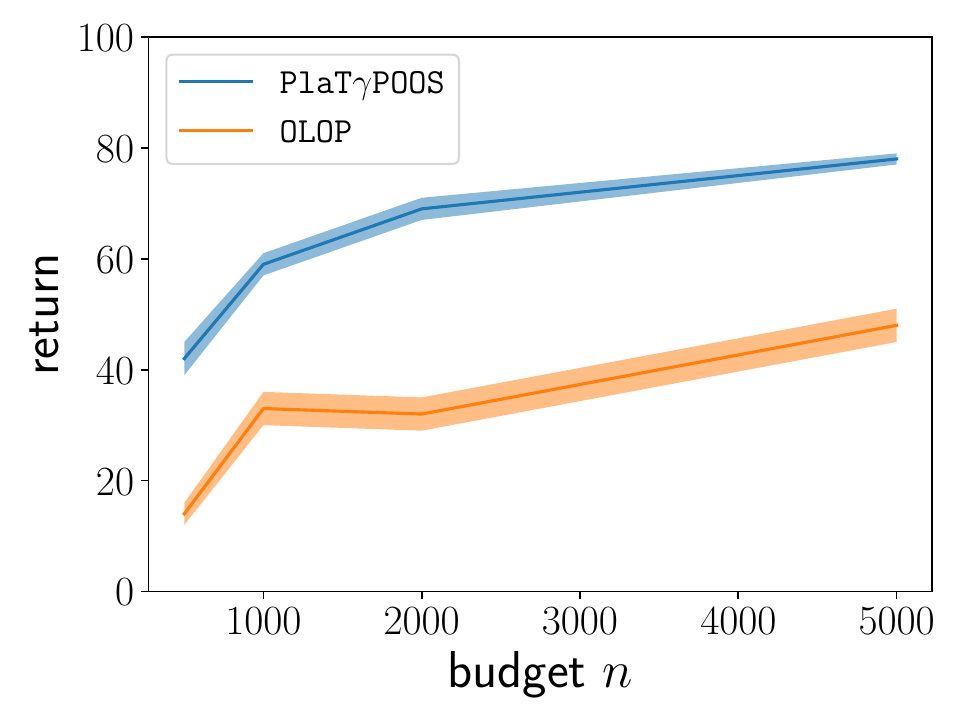}
	\\

	\includegraphics[width=0.325\textwidth]{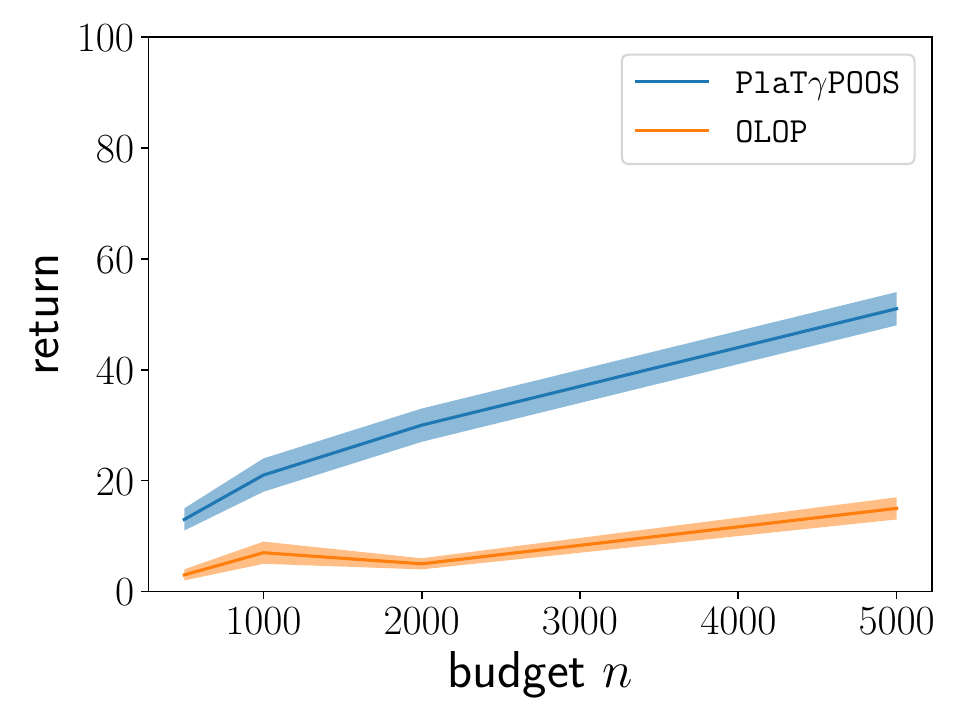}
		\includegraphics[width=0.325\textwidth]{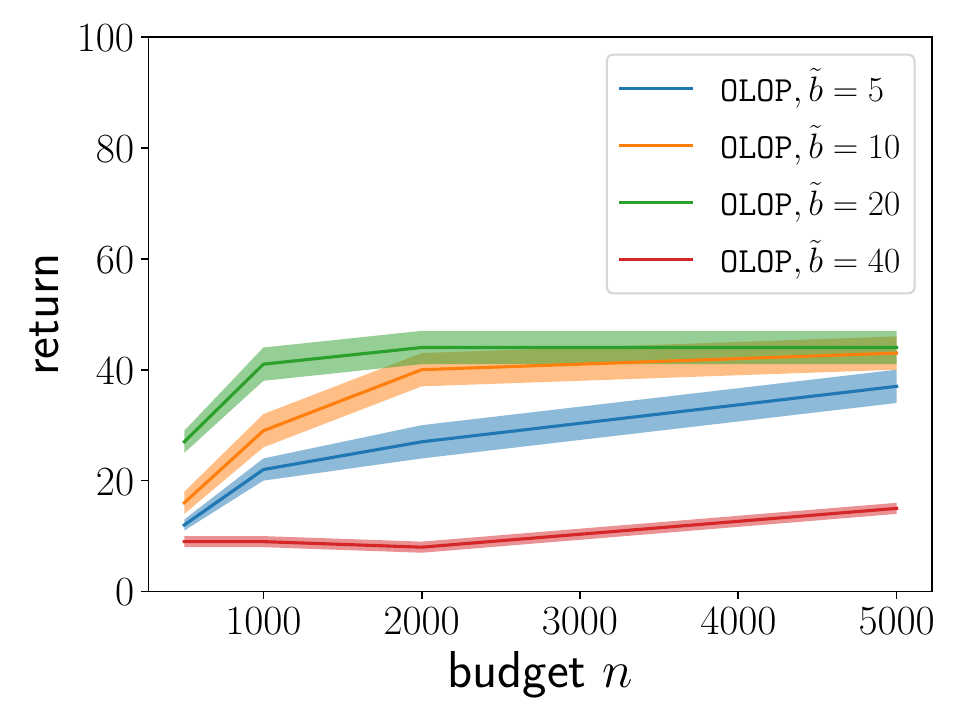}
		\includegraphics[width=0.325\textwidth]{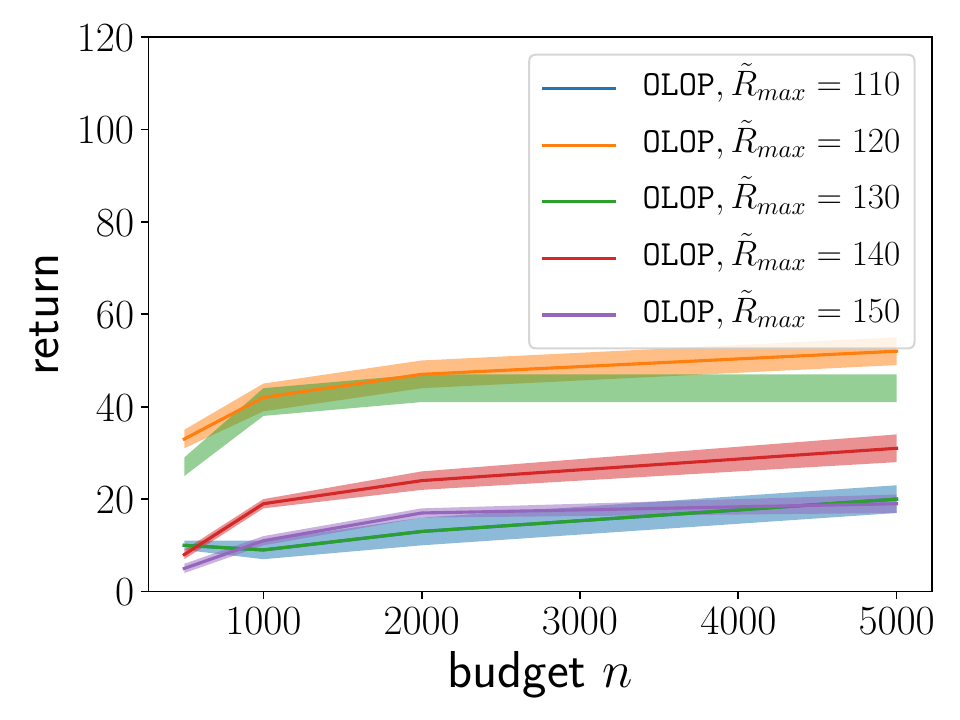}
	\vspace{-0.1in}
	\caption{ 
		\emph{Top and bottom left:} Average cumulative discounted return collected by \OLOP and \platypoos with different range of  noise, $b=1$ (top left),  $b=10$ (top center), $b=20$, (top right), and $b=50$ (bottom left).
		\emph{Bottom center:}  the sensitivity of \OLOP to different $\tilde R_{\max}$ parameters. 	\emph{Bottom right:}  the sensitivity of \OLOP to different range of the input noise $\tilde b$ as parameters while the true $b$ is set to $10$.}
		\vspace{-0.05in}
	\label{Sin}
\end{figure*}
	\vspace{-0.05in}
\paragraph{Adaptation to
the deterministic case and $\kappa = 1$}
\platypoos adapts to the branching factor $\kappa$ of the problem  that under low noise conditions, 
leads to an exponentially decreasing simple regret $r_n^{\platypoos} = \cO(\nu\rho^{\sqrt{n}}).$ This is a light-years improvement over \OLOP for these conditions, as \OLOP 's regret is at best $r_n^{\platypoos} =\cO(n^{-1/2})$. This result is possible because \platypoos  explores much deeper than $\OLOP$, as its maximal depth is of order $n$.
Actually, in most scenarios, the actual larger depth explored will be of order $\sqrt{n}$ due to sampling limitations. On the other hand, \OLOP can only go $\log n$ deep.

Moreover, $\kappa = 1$ is a common case in planning. Indeed, as discussed by~\citet{bubeck2010open}, $\kappa = 1$ is  equivalent to having near-optimal dimension $d=0$ in an optimization task~\cite{munos2014from} which is a common value as shown by~\citet{valko2013stochastic}.
Therefore, we expect the case when $\discount\kappa^2\leq 1$, that is, where we get the  most significant improvement other \OLOP, to be the most common in practice.
\paragraph{The reset condition}
As discussed in Section~\ref{sec:deterdeter}, in the stochastic reward case, the effect of the reset condition  affects,  \platypoos as follows.
First, all polynomial rates stay the same. Next, only the exponential rates change. 
A~$\rho^n$ rate without the condition becomes a $\rho^{\sqrt{n}}$ with it.
Next, a~$\rho^{\sqrt{n}}$ rate becomes a $\rho^{n^{1/3}}$ one.
\vfil
\section{Numerical experiments}
%
%

In this section, we empirically illustrate the benefits of \platypoos. We chose a simple MDP, shown 
in Figure~\ref{fig:mdp2}. In this MDP, a state $\stateMDP \triangleq ({\rm bin},d)$ is a pair of a binary variable~${\rm bin}$ and a non-negative integer $d$.
The MDP has two actions that are also binary. If ${\rm bin}\neq a$, the base reward is $2$,
in which case,  the next state is $(a,0)$.  Otherwise, if ${\rm bin}= a,$  then $r=d$ and the next state is $(a,d+1)$. The reward is then shifted by adding 100 to it so that the noises with different ranges can be added on top without making the reward negative.

\begin{figure}[H]
\centering
\begin{tikzpicture}[node distance=3.5cm,on grid,very thick]
  \node[state]   (q_0)                {0};
  \node[state] (q_2)  [right=of q_0]{1};
  \path[->] (q_0) 
                  edge [loop above, blue, align=center]   node  {$r=$\#consecutive\\visits} ()
                  edge [bend left=45, red] node [above] {$r=2$} (q_2)
            (q_2)  edge [bend left, red]   node [below] {$r=2$} (q_0)
            (q_2)  edge [loop above, blue,, align=center]   node       {$r=$\#consecutive\\ visits} ();
\end{tikzpicture}
\vspace{-0.1in}
\caption{MDP used for our experiments} 
		\label{fig:mdp2}
		\vspace{-0.1in}
\end{figure}
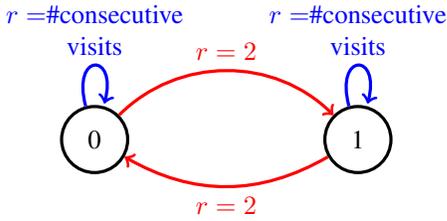

The initial state is $(0,0)$. Therefore, the agent has a choice. It can, for instance, remain in the same binary state {\rm bin}, starting with a null reward but sees its instant reward growing with time if it keeps taking the  same action in the future. Alternatively, it could greedily switch to the other binary state ${\rm bin}$ and obtain a reward of $2$ but delaying the hope of obtaining growing reward as in the first scenario. 
We set  $\discount=0.95$. Therefore, $\Rmax\approx130$.

Figure~\ref{Sin} reports the results. All the figures show the cumulative discounted return
collected by \OLOP and \platypoos after having interacted for $20$ steps with the MDP, having chosen each time an action following their planning strategy and then being transferred to the state resulting of applying the recommended action in the current state; therefore also collecting a reward that is composing the final return. 

Note that the return reported are shifted in order to not take into account the  fixed $100$ part of each the reward.
The figures in the top row, as well as the figure at the bottom left, reports the comparison between the two returns of \OLOP and \platypoos for different ranges of noise $b$. \platypoos is systematically outperforming \OLOP while in this case \OLOP is given the correct  $\tilde R_{max}$ as input,  $\tilde R_{max}= \Rmax$ and the correct range of the noise $\tilde b,$ that is $\tilde b = b$.

In  Figure~\ref{Sin}, bottom center and  right, we illustrate the sensitivity of \OLOP to misleading input parameters. Notice that  the performance of \OLOP is very vulnerable to these misspecifications while \platypoos is not using such inputs.
\paragraph{Acknowledgments}
The research presented was supported by European CHIST-ERA project DELTA, French Ministry of
Higher Education and Research, Nord-Pas-de-Calais Regional Council,
Inria and Otto-von-Guericke-Universit\"at Magdeburg associated-team north-European project Allocate, and French National Research Agency project BoB (grant n.ANR-16-CE23-0003), 
FMJH Program PGMO with the support of this program from Criteo.
\clearpage
\bibliography{perfectlibrary,references,b,library,biblio}

\begin{thebibliography}{27}
\providecommand{\natexlab}[1]{#1}
\providecommand{\url}[1]{\texttt{#1}}
\expandafter\ifx\csname urlstyle\endcsname\relax
  \providecommand{\doi}[1]{doi: #1}\else
  \providecommand{\doi}{doi: \begingroup \urlstyle{rm}\Url}\fi

\bibitem[Bartlett et~al.(2019)Bartlett, Gabillon, and
  Valko]{bartlett2019simple}
Bartlett, P.~L., Gabillon, V., and Valko, M.
\newblock
  \href{http://researchers.lille.inria.fr/~valko/hp/serve.php?what=publications/bartlett2019simple.pdf}{{A
  simple parameter-free and adaptive approach to optimization under a minimal
  local smoothness assumption}}.
\newblock In \emph{Algorithmic Learning Theory (ALT)}, 2019.

\bibitem[Bertsekas \& Tsitsiklis(1996)Bertsekas and
  Tsitsiklis]{bertsekas1996neuro-dynamic}
Bertsekas, D. and Tsitsiklis, J.
\newblock
  \href{https://books.google.co.uk/books/about/Neuro_dynamic_Programming.html?id=WxCCQgAACAAJ&source=kp_book_description&redir_esc=y}{\emph{{Neuro-dynamic
  programming}}}.
\newblock Athena Scientific, Belmont, MA, 1996.

\bibitem[Bubeck \& Munos(2010)Bubeck and Munos]{bubeck2010open}
Bubeck, S. and Munos, R.
\newblock \href{http://sbubeck.com/COLT10_BM.pdf}{{Open-loop optimistic
  planning}}.
\newblock In \emph{Conference on Learning Theory (COLT)}, 2010.

\bibitem[Bubeck et~al.(2011)Bubeck, Munos, Stoltz, and
  Szepesv{\'{a}}ri]{bubeck2011x}
Bubeck, S., Munos, R., Stoltz, G., and Szepesv{\'{a}}ri, C.
\newblock
  \href{http://www.jmlr.org/papers/volume12/bubeck11a/bubeck11a.pdf}{{$\mathcal{X}$-armed
  bandits}}.
\newblock \emph{Journal of Machine Learning Research}, 12:\penalty0 1587--1627,
  2011.

\bibitem[Bu\c{s}oniu \& Munos(2012)Bu\c{s}oniu and
  Munos]{busoniu2012optimistic}
Bu\c{s}oniu, L. and Munos, R.
\newblock
  \href{http://proceedings.mlr.press/v22/busoniu12/busoniu12.pdf}{{Optimistic
  planning for Markov decision processes}}.
\newblock In \emph{International Conference on Artificial Intelligence and
  Statistics (AISTATS)}, 2012.

\bibitem[Carpentier \& Locatelli(2016)Carpentier and Locatelli]{Carpentier16TB}
Carpentier, A. and Locatelli, A.
\newblock \href{http://proceedings.mlr.press/v49/carpentier16.pdf}{Tight
  (lower) bounds for the fixed budget best-arm identification bandit problem}.
\newblock In \emph{Conference on Learning Theory (COLT)}, 2016.

\bibitem[Coquelin \& Munos(2007)Coquelin and Munos]{coquelin2007bandit}
Coquelin, P.-A. and Munos, R.
\newblock \href{https://arxiv.org/pdf/1408.2028.pdf}{{Bandit algorithms for
  tree search}}.
\newblock In \emph{Conference on Uncertainty in Artificial Intelligence (UAI)}, 2007.

\bibitem[Coulom(2007)]{coulom2007efficient}
Coulom, R.
\newblock \href{https://hal.inria.fr/inria-00116992/document}{{Efficient
  selectivity and backup operators in Monte-Carlo tree search}}.
\newblock \emph{Computers and games}, 4630:\penalty0 72–83, 2007.

\bibitem[Feldman \& Domshlak(2014)Feldman and Domshlak]{feldman2014simple}
Feldman, Z. and Domshlak, C.
\newblock \href{http://dx.doi.org/10.1613/jair.4432}{{Simple regret
  optimization in online planning for Markov decision processes}}.
\newblock \emph{Journal of Artificial Intelligence Research}, 2014.

\bibitem[Gabillon et~al.(2012)Gabillon, Ghavamzadeh, and
  Lazaric]{gabillon2012best}
Gabillon, V., Ghavamzadeh, M., and Lazaric, A.
\newblock \href{https://hal.inria.fr/hal-00747005v1/document}{{Best-arm
  identification: A unified approach to fixed budget and fixed confidence}}.
\newblock In \emph{Neural Information Processing Systems (NeurIPS)}, 2012.

\bibitem[Gelly et~al.(2006)Gelly, Yizao, Munos, and
  Teytaud]{gelly2006modifications}
Gelly, S., Yizao, W., Munos, R., and Teytaud, O.
\newblock \href{https://hal.inria.fr/inria-00117266}{{Modification of UCT with
  patterns in Monte-Carlo Go}}.
\newblock Technical report, Inria, 2006.

\bibitem[Grill et~al.(2016)Grill, Valko, and Munos]{grill2016blazing}
Grill, J.-B., Valko, M., and Munos, R.
\newblock
  \href{https://papers.nips.cc/paper/6253-blazing-the-trails-before-beating-the-path-sample-efficient-monte-carlo-planning}{{Blazing
  the trails before beating the path: Sample-efficient Monte-Carlo planning}}.
\newblock In \emph{Neural Information Processing Systems (NeurIPS)}, 2016.

\bibitem[Hoorfar \& Hassani(2008)Hoorfar and Hassani]{Hoorfar08IO}
Hoorfar, A. and Hassani, M.
\newblock
  \href{https://www.emis.de/journals/JIPAM/images/107_07_JIPAM/107_07.pdf}{Inequalities
  on the lambert w function and hyperpower function}.
\newblock \emph{Journal of Inequalities in Pure and Applied Mathematics},
  9\penalty0 (2):\penalty0 5--9, 2008.

\bibitem[Hren \& Munos(2008)Hren and Munos]{hren2008optimistic}
Hren, J.-F. and Munos, R.
\newblock
  \href{https://hal.archives-ouvertes.fr/hal-00830182/document}{{Optimistic
  planning of deterministic systems}}.
\newblock In \emph{European Workshop on Reinforcement Learning}, 2008.

\bibitem[Kaufmann \& Koolen(2017)Kaufmann and Koolen]{kaufmann2017monte}
Kaufmann, E. and Koolen, W.~M.
\newblock \href{https://arxiv.org/pdf/1706.02986.pdf}{{Monte-carlo tree search
  by best-arm identification}}.
\newblock In \emph{Neural Information Processing Systems (NeurIPS)}, 2017.

\bibitem[Kocsis \& Szepesv{\'{a}}ri(2006)Kocsis and
  Szepesv{\'{a}}ri]{kocsis2006bandit}
Kocsis, L. and Szepesv{\'{a}}ri, C.
\newblock \href{http://ggp.stanford.edu/readings/uct.pdf}{{Bandit-based
  Monte-Carlo planning}}.
\newblock In \emph{European Conference on Machine Learning (ECML)}, 2006.

\bibitem[Leurent \& Maillard(2019)Leurent and Maillard]{leurent2019practical}
Leurent, E. and Maillard, O.-A.
\newblock \href{https://arxiv.org/pdf/1904.04700.pdf}{{Practical Open-Loop
  Optimistic Planning}}.
\newblock \emph{arXiv preprint arXiv:1904.04700}, 2019.

\bibitem[Munos(2011)]{munos2011optimistic}
Munos, R.
\newblock
  \href{https://papers.nips.cc/paper/4304-optimistic-optimization-of-a-deterministic-function-without-the-knowledge-of-its-smoothness.pdf}{{Optimistic
  optimization of deterministic functions without the knowledge of its
  smoothness}}.
\newblock In \emph{Neural Information Processing Systems (NeurIPS)}, 2011.

\bibitem[Munos(2014)]{munos2014from}
Munos, R.
\newblock \href{https://hal.archives-ouvertes.fr/hal-00747575v5/document}{{From
  bandits to Monte-Carlo tree search: The optimistic principle applied to
  optimization and planning}}.
\newblock \emph{Foundations and Trends in Machine Learning}, 7(1):\penalty0
  1--130, 2014.

\bibitem[Orabona \& P{\'{a}}l(2018)Orabona and P{\'{a}}l]{orabona2018scale}
Orabona, F. and P{\'{a}}l, D.
\newblock \href{http://dx.doi.org/10.1016/j.tcs.2017.11.021}{{Scale-free online
  learning}}.
\newblock \emph{Theoretical Computer Science}, 2018.

\bibitem[Orabona \& Tommasi(2017)Orabona and Tommasi]{orabona2017training}
Orabona, F. and Tommasi, T.
\newblock
  \href{http://papers.neurips.cc/paper/6811-training-deep-networks-without-learning-rates-through-coin-betting.pdf}{{Training
  deep networks without learning rates through coin betting}}.
\newblock In \emph{Neural Information Processing Systems (NeurIPS)}, 2017.

\bibitem[Puterman(1994)]{puterman1994markov}
Puterman, M.~L.
\newblock
  \href{http://www.amazon.ca/exec/obidos/redirect?tag=citeulike09-20&amp;path=ASIN/0471619779}{\emph{{Markov
  Decision Processes: Discrete Stochastic Dynamic Programming}}}.
\newblock John Wiley {\&} Sons, New York, NY, 1994.

\bibitem[Ross et~al.(2013)Ross, Mineiro, and Langford]{ross2013normalized}
Ross, S., Mineiro, P., and Langford, J.
\newblock \href{http://arxiv.org/abs/1305.6646}{{Normalized online learning}}.
\newblock In \emph{Conference on Uncertainty in Artificial Intelligence (UAI)},  2013.

\bibitem[Shah et~al.(2019)Shah, Xie, and Xu]{shah2019reinforcement}
Shah, D., Xie, Q., and Xu, Z.
\newblock \href{https://arxiv.org/pdf/1902.05213.pdf}{{On reinforcement
  learning using Monte-Carlo tree search with supervised learning:
  Non-asymptotic analysis}}.
\newblock \emph{arXiv preprint arXiv:1902.05213}, 2019.

\bibitem[Silver et~al.(2016)Silver, Huang, Maddison, Guez, Sifre, van~den
  Driessche, Schrittwieser, Antonoglou, Panneershelvam, Lanctot, Dieleman,
  Grewe, Nham, Kalchbrenner, Sutskever, Lillicrap, Leach, Kavukcuoglu, Graepel,
  and Hassabis]{silver2016mastering}
Silver, D., Huang, A., Maddison, C.~J., Guez, A., Sifre, L., van~den Driessche,
  G., Schrittwieser, J., Antonoglou, I., Panneershelvam, V., Lanctot, M.,
  Dieleman, S., Grewe, D., Nham, J., Kalchbrenner, N., Sutskever, I.,
  Lillicrap, T., Leach, M., Kavukcuoglu, K., Graepel, T., and Hassabis, D.
\newblock
  \href{http://www0.cs.ucl.ac.uk/staff/d.silver/web/Publications_files/unformatted_final_mastering_go.pdf}{{Mastering
  the game of Go with deep neural networks and tree search}}.
\newblock \emph{Nature}, 529\penalty0 (7587):\penalty0 484--489, 2016.

\bibitem[Sz{\"{o}}r{\'{e}}nyi et~al.(2014)Sz{\"{o}}r{\'{e}}nyi, Kedenburg, and
  Munos]{szorenyi2014optimistic}
Sz{\"{o}}r{\'{e}}nyi, B., Kedenburg, G., and Munos, R.
\newblock
  \href{https://papers.nips.cc/paper/5368-optimistic-planning-in-markov-decision-processes-using-a-generative-model.pdf}{{Optimistic
  planning in Markov decision processes using a generative model}}.
\newblock In \emph{Neural Information Processing Systems (NeurIPS)}, 2014.

\bibitem[Valko et~al.(2013)Valko, Carpentier, and Munos]{valko2013stochastic}
Valko, M., Carpentier, A., and Munos, R.
\newblock \href{http://proceedings.mlr.press/v28/valko13.pdf}{{Stochastic
  simultaneous optimistic optimization}}.
\newblock In \emph{International Conference on Machine Learning (ICML)}, 2013.

\end{thebibliography}
\appendix
\onecolumn

\section{On the branching factor}\label{otbf}
\propo*
\begin{proof}
 For any global optimum and for $h\geq 0$, we prove that
\[\mathcal N^v_h(\epsilon) \leq \mathcal N^u_h\pa{\epsilon+ \frac{\discount^h}{1-\discount}}\cdot\]
Let a node $\action\in\actionSpace^h$ be such that  	$v(\action) \geq  v^\star   -  \epsilon.$ 
Then, we have
\[u(\action) \geq  v(\action) - \frac{\discount^h}{1-\discount}
\geq  v^\star   -  \epsilon - \frac{\discount^h}{1-\discount}\cdot\]
Similarly, for $h\geq 0$, we have that for any global optimum, 
\[\mathcal N^u_h(\epsilon) \leq \mathcal N^v_h\pa{\epsilon+\frac{\discount^h}{1-\discount}}\!\cdot\]
Using Definition~\ref{def:neardim}, we get the claimed result.	
\end{proof}

\section{{\platypoos} is not using a budget larger than $n+1$ }\label{plabud} 
\label{app:budstro}
Notice that for any given depth $h\in[1:\hmax]$, \platypoos never uses more evaluations than $(\pmax +1) \frac{\hmax}{h}$ because
\begin{align*}
&\sum_{p=0}^{\left\lfloor\log_2(\hmax/\left\lceil h \discount^{2h}\right\rceil )\right\rfloor}
\left\lfloor \frac{\hmax}{h \left\lceil h 2^p\discount^{2h}\right\rceil}\right\rfloor \left\lceil h 2^p\discount^{2h}\right\rceil\leq 
\pa{\left\lfloor\log_2(\hmax/\left\lceil h \discount^{2h}\right\rceil )\right\rfloor +1} \frac{\hmax}{h}\cdot
\end{align*}
Summing over the depths,  \platypoos never uses more evaluations than the budget $\timeHorizon+1$ during its depth exploration as 
\begin{align*}
1+(\pmax +1)\sum_{h=1}^{\hmax}\left\lfloor \frac{\hmax}{h}\right\rfloor
&\leq 1+(\pmax +1)\hmax\sum_{h=1}^{\hmax}  \frac{1}{h}\\
&= 1+\hmax  \bar\log(\hmax)(\pmax +1)
\leq 1+\hmax  (\pmax +1)^2\\
&\leq \frac{\timeHorizon}{2}+1.
\end{align*}
We need to add the additional evaluation for the cross-validation at the end, 
\begin{align*}
\sum_{p=0}^{\pmax}
\sum_{t=0}^{\hmax}
\frac{(t+1)\discount^{2t}\hmax}{(1-\discount^2)^2}
\leq
\sum_{p=0}^{\pmax}
\left\lfloor \frac{\timeHorizon}{2(\log_2\timeHorizon+1)^2}\right\rfloor
\leq \frac{\timeHorizon}{2}\cdot
\end{align*}
Therefore, the total budget is never more than 
$\timeHorizon/2 + \timeHorizon/2 + 1=\timeHorizon+1$.
Again, notice we use the budget of $\timeHorizon+1$ only for the notational convenience, we
could also use $\timeHorizon/4$ for the evaluation in the end to fit under $\timeHorizon$. Nonetheless, it's important that the amount of openings is \emph{linear} in $\timeHorizon$.

\section{Proofs of the lemmas}\label{lemproofs} %
We first define favorable event $\xi$ and prove that
it holds with high probability. 
\begin{restatable}{lemma}{lemevent}\label{lem:event}
\label{l:event}
	Let  $\mathcal{C}$ be the set of sequence of actions evaluated by {\platypoos} during one of its runs. $\mathcal{C}$ is a random quantity.
	Let $\xi$ be the event under which all average
	estimates for the reward of the state-action pairs receiving at least one evaluation from {\platypoos} are within their confidence interval, then  $P(\xi)\geq 1 - \delta$, where
	\begin{align*}
	\xi
	\triangleq
	\left\{\! \forall \action \in \mathcal{C}, \forall \depth \in [2\!:\!\depth(\action)], \,p\in[0\!:\!\pmax]:
	\text{if\ }\pullsNumber_{\action_{[h]}}\geq  \left\lceil (h-1) 2^p\discount^{2(h\!-\!1)}\!\right\rfloor\!,
	\text{then}
	\left|\hat u(\action)\! -\! u(\action)    \right| \leq b\sqrt{    \frac{\pmax\log(4\timeHorizon/\delta)}{2^{p+1}}}
	\right\}\!\cdot
\end{align*}
	%
\end{restatable}
\begin{proof}
	The  idea of the proof follows the  
	line of proof of the  statement given for
	\StoSOO \citep{valko2013stochastic}. The crucial point is that 
	while we have potentially exponentially many combinations 
	of cells that can be evaluated, given any particular execution 
	we need to consider only a polynomial number of estimators, 	$m$,
	for which we can use a Azuma-Hoeffding concentration inequality.

	We denote  $\forall \depth \in [0:\hmax], \,p\in[0:\pmax]:$
	$\action^{i,h,p}\in \mathcal{C},$ the $i$-th evaluated node of depth $h$ such that $\forall t \in [2:\depth],  \pullsNumber_{\action^{i,h,p}_{[t]}}\geq \left\lceil (t-1) 2^p\discount^{2(t-1)}\right\rceil$. Note that in \platypoos we have $\pullsNumber_{\action^{i,h,p}_{[1]}}=\hmax $.
	
	Though $\action^{i,h,p}$ is random, we study the quantity
	 $\left|\hat u(\action^{i,h,p}) - u(\action^{i,h,p})    \right| $.
	We recall that
	 \begin{align}
	\hat u(\action^{i,h,p}) - u(\action^{i,h,p}) &= \sum^{\depth-1}_{t=0}
	\discount^t (\hat \reward_t( \action^{i,h,p} )-\reward_t(  \action^{i,h,p} ))\\
	 &= \sum^{\depth-1}_{t=0}
	\discount^t \sum^{\pullsNumber_{\action^{i,h,p}_{[t+1]}}}_{s=0}
	\frac{\hat \reward^{i,h,p}_{t,s} -\reward_t(  \action^{i,h,p} )}{\pullsNumber_{\action^{i,h,p}_{[t+1]}}} 	
	\end{align}
	
	 This quantity is composed of the elements $\hat \reward^{i,h,p}_{t,s} -\reward_t(  \action^{i,h,p})$ that form a martingale.
	
Therefore using a  Azuma-Hoeffding concentration inequality with a union bound already on the values of T we have

\[
 \Pro\left(	\hat u(\action^{i,h,p}) - u(\action^{i,h,p})\leq b \sqrt{\sum^{\depth-1}_{t=0} \frac{\discount^{2t}\log(\pmax/\delta)}{2\pullsNumber_{\action^{i,h,p}_{[t+1]}}} } \right) \\
\geq 1-\delta/\pmax
\]

Moreover we have  for all $h\geq t>1$,
\begin{align}
&  \frac{\discount^{2t}}{\pullsNumber_{\action^{i,h,p}_{[t+1]}}} 
\leq  \frac{\discount^{2t}}{\left\lceil t 2^p\discount^{2t}\right\rceil}  
  \leq  \frac{\discount^{2t}}{  t 2^p\discount^{2t}}  =  \frac{1}{  t 2^p} 
\end{align}
For $t=0$, $\frac{\discount^{2t}}{\pullsNumber_{\action^{i,h,p}_{[t+1]}}} 
= \frac{1}{\hmax} \leq = \frac{1}{2^p} $ for all $p\leq \pmax$.

Therefore we have

\[
\Pro\left(	\hat u(\action^{i,h,p}) - u(\action^{i,h,p})\leq b \sqrt{ \frac{\log\hmax\log(?/\delta)}{2^{p+1}} } \right) \\
\geq 1-\delta/?
\]
Then we had an extra union bound other all cells that is bounded by $n$

\end{proof}%

\begin{restatable}{lemma}{restahstarSto}\label{lem:hstarSto}
	For any planning problem with associated $(\nu,\rho)$ (see Property~\ref{as:smoothu}), on event $\xi$, for any depths $h \in \left[\hmax\right]$, for any $p\in[0:\left\lfloor\log_2(\hmax/(h^2\discount^{2h}))\right\rfloor]$, we have 	$\depthOp_{h,p}= h$ if conditions (1) and (2) simultaneously hold true.\\
	(1)	 $b \sqrt{\log(4n/\delta) / 2^{p+1}} \leq \nu\rho^h$
	\\
	(2)	 For all $h'\in[h], \hmax / \pa{h'\left\lceil h' 2^{p}\discount^{2h'}\right\rceil}\geq  C\branchF(\nu,\rho)^{h'}$.
	\\
	Finally we
	have $\depthOp_{0,p}=0$.
\end{restatable}
\begin{proof}
	We place ourselves on 
	event $\xi$ defined in Lemma~\ref{l:event} and for which we  proved  that $P(\xi)\geq 1 - \delta$.  We fix $p$. 
	
		We prove the statement of the lemma, given that event $\xi$ holds, by induction in the following sense. For a given $h$ and $p$, we assume the hypotheses of the lemma for that $h$ and $p$ are true and we prove by induction that $\depthOp_{h',p}= h'$ for $h'\in[h]$. \\[.1cm]    
	$1^\circ$ For $h= 0$, we trivially have that $\depthOp_{h,p} \geq 0$.\\     
	$2^\circ$ Now consider $h'>0$, and assume $\depthOp_{h'-1,p}= h'-1$ with the objective to prove that $\depthOp_{h',p}= h'$.

	Therefore, at the end of the processing of depth $h'-1$, during which we were opening the nodes of depth $h'-1$ we managed to open an optimal node  that we denote $a^{\star,h'-1}\in\actionSpace^{\star,h'-1}$. Moreover if we consider all the sequence of actions $b$ that one can build by appending any action in $A$ to $a^{\star,h'-1}\in\actionSpace^{\star,h'-1}$, we have for all such $b$ that $ \pullsNumber_{b_{[t]}}\geq  \left\lceil (t-1)2^p\discount^{t-1} \right\rceil $ for $t\in[h']$.


	Note that by definition there exist an optimal infinite sequence of actions $a^\star\in\actionSpace^{\star}$ such that $a^{\star,h'-1}= a_{[h'-1]}^\star$

	
	During phase ${h'}$ the $\left\lfloor \hmax / \pa{h'\left\lceil h' 2^p\discount^{2h'}\right\rceil}\right\rfloor $ evaluated nodes from $A^{h'-1}$  
	 with highest values $\{\hat u(a^{h'-1,i})\}_{h'-1,i}$  are opened. 

	For the purpose of contradiction, let us assume that $a_{[h']}^\star$ is not one of them. This would mean that there exist at least $\left\lfloor \hmax / \pa{h'\left\lceil h' 2^p\discount^{2h'}\right\rceil}\right\rfloor $ 	
	nodes from $\actionSpace^{h'}$, distinct from $a_{[h']}^\star$,  satisfying 
	$\hat u(a^{h',i})
	\geq
	\hat u(a_{[h']}^\star )$  and each verifying $ \pullsNumber_{a^{h',i}_{[t]}}\geq  \left\lceil t2^p\discount^{t} \right\rceil$ for  $t\in[h']$. 
	This means that, for these nodes we have:
	$u(a_{h',i})+\nu\rho^{h'}
	\geq
	u(a_{h',i}) + \nu\rho^{h}
	\stackrel{\textbf{(a)}}{\geq} 
	u(a_{h',i}) + b\sqrt{\log(4\timeHorizon/\delta) / 2^{p+1}}
	\stackrel{\textbf{(b)}}{\geq} 
	\hat u(a_{h',i})
	\geq
	\hat u(a_{h',i^\star_{h'}})
	\stackrel{\textbf{(b)}}{\geq} 
	u(a_{h',i^\star_{h'}})- b\sqrt{\log(4\timeHorizon/\delta) / 2^{p+1}}
	\stackrel{\textbf{(a)}}{\geq} 
	u(a_{h',i^\star_{h'}})- \nu\rho^{h}
	\geq 
	u(a_{h',i^\star_{h'}})- \nu\rho^{h'}$,
	where~\textbf{(a)} is by assumption of the lemma, 
	\textbf{(b)} is because $\xi$ holds.
	As $u(a_{h',i^\star_{h'}}) \geq v^\star -\nu\rho^{h'}$ by Proposition~\ref{as:smoothu}, 
	this means we have $ \mathcal N^u_{h'}(3\nu\rho^{h'})\geq \left\lfloor \hmax / \pa{h' \left\lceil h' 2^p\discount^{h'} \right \rceil }\right\rfloor+1$ (the $+1$ is for $a_{h,i_h^\star}$). 
	%
	However, by assumption of the lemma $
	\hmax / \pa{h'\left\lceil h' 2^p\discount^{2h'}\right\rceil}\geq
	C\branchF(\nu,\rho)^{h'}$. 
	It follows that in general $ \mathcal N^u_{h'}(3\nu\rho^{h'})> \left\lfloor C\branchF(\nu,\rho)^{h'}\right\rfloor$. 
	This leads to having a contradiction with the $\branchF^u(\nu,\rho)$ with associated constant $C$
	as defined in~Definition~\ref{def:neardim}.
	Indeed, the condition $\mathcal N^u_{h'}(3\nu\rho^{h'}) \leq     C\branchF(\nu,\rho)^{h'}$ in 
	Definition~\ref{def:neardim} is equivalent to the condition
	$\mathcal N^u_{h'}(3\nu\rho^{h'}) \leq \left\lfloor    C\branchF(\nu,\rho)^{h'}\right\rfloor$ as  $\mathcal N^u_{h'}(3\nu\rho^{h'})$ is an integer. 
	%

	%
	%
\end{proof}

\restahstarStoTwo*

\begin{proof}
To prove this statement we just need to show that we verify the hypotheses of Lemma~\ref{lem:hstarSto}. This means we need to prove that for all $h'\in[h], \hmax / \pa{h'\left\lceil h' 2^{p}\discount^{2h'}\right\rceil}\geq  C\branchF(\nu,\rho)^{h'}$.

 We first consider the case 2) where $ h2^p\discount ^h\geq 1$.  If $h=1$ we already know $\depthOp_{0,p} \geq 0$. Let us now look at the case $h>1$.
First notice that $ h2^p\discount ^{2h}\geq 1$ gives $ (h-1)2^p\discount ^{2(h-1)}\geq 1$.
 If $\discount^{2}\branchF^u\geq 1$ 
we have that for all $h'\in[h-1]$, 
	\begin{align*}
		\frac{\hmax}{h'^2 2^{p+1}}\geq 
	\frac{\hmax}{h^2 2^{p+1}}
	\geq  C\left(\discount^2\branchF(\nu,\rho)\right)^h
	\geq  C\left(\discount^2\branchF(\nu,\rho)\right)^{h'}
		\end{align*}
	
			%
If $h>1$, and if $\discount^{2}\branchF^u\leq 1$ 
we have that for all $h'\in[h-1]$, 
	\begin{align*}
\frac{\hmax}{h'2^{p+1}}\geq 
\frac{\hmax}{h2^{p+1}}
\geq  C
\geq C\left(\discount^2\branchF(\nu,\rho)\right)^{h'}.
\end{align*}
For both $\discount^{2}\branchF^u\leq 1$ and $\discount^{2}\branchF^u\geq 1$,
we then have,
\begin{align*}
	\frac{\hmax}{h'\left\lceil h' 2^p\discount^{2h'}\right\rceil}\geq \frac{\hmax}{h' h' 2^{p+1}\discount^{2h'}} 
\end{align*}	
	as  $h2^p\discount ^{2h}\geq 1$.

For both $\discount^{2}\branchF^u\leq 1$ and $\discount^{2}\branchF^u\geq 1$, the previous equations  mean that for $h'\in[h-1]$, $h'$ verifies $\:	\hmax / {h'^2 2^{p+1}\discount^{2h'}}
\geq  C\branchF(\nu,\rho)^{h'}\geq 1$. Therefore $p\leq\left\lfloor\log_2(\hmax/(h'^2\discount^{2h'}))\right\rfloor$.

We now consider case 1) where $h2^p\discount ^{2h}\leq 1$.
We prove by induction that for all $h'\in[h],  \hmax / \pa{h'\left\lceil h' 2^{p}\discount^{2h'}\right\rceil}\geq  C\branchF(\nu,\rho)^{h'}$. 
\begin{enumerate}
\item[$1^\circ$] By assumption of the lemma we say: $\hmax / \pa{h\left\lceil h 2^{p}\discount^{2h}\right\rceil}\geq C\branchF(\nu,\rho)^{h}$ 
\item[$2^\circ$] We further assume $\hmax / \pa{h'\left\lceil h' 2^{p}\discount^{2h'}\right\rceil}\geq  C\branchF(\nu,\rho)^{h'}$is true for some $h'\leq h$  with $h'2^p\discount ^{2h'}\leq 1$ 
\end{enumerate}
We want to prove that either:
\begin{enumerate}
\item[both] $(h'-1)2^p\discount ^{2(h'-1)}\leq 1$
\item[and] $\frac{\hmax}{(h'-1)\left\lceil (h'-1) 2^{p}\discount^{2(h'-1)}\right\rceil}
\geq  C\branchF(\nu,\rho)^{h'-1}$ 
\item[or]
$(h'-1)2^p\discount ^{2(h'-1)}\geq 1$
\end{enumerate}
then $\hmax / \pa{(h'')\left\lceil (h'')2^{p}\discount^{2(h'')}\right\rceil}\geq  C\branchF(\nu,\rho)^{h''}$ is already true for all $h''\in[h']$.
If $\left\lceil (h'-1)2^p\discount ^{2(h'-1)}\right\rceil = 1$ then we have  
\begin{align*}
\frac{\hmax}{(h'-1) }\geq 
\frac{\hmax}{h' }
\geq  C\branchF(\nu,\rho)^{h'}
\geq  C\branchF(\nu,\rho)^{h'-1}
\end{align*}
If $\left\lceil (h'-1)2^p\discount ^{2(h'-1)}\right\rceil > 1$, then we have that
\begin{enumerate}
\item[]$\hmax / \pa{h'\left\lceil (h'-1) 2^{p}\discount^{2(h'-1)}\right\rceil}\geq$
\item[]$\hmax / \pa{(h'-1)^2 2^{p+1}\discount^{2(h'-1)}}\geq  C\branchF(\nu,\rho)^{h'-1}$
\end{enumerate}

Using this inequality we can now use Case 2) to have that:  $\hmax / \pa{(h'')\left\lceil (h'') 2^{p}\discount^{2(h'')}\right\rceil}\geq  C\branchF(\nu,\rho)^{h''}$ is already true for all $h''\in[h']$.

 The previous equations  mean that for $h'\in[h-1]$, $h'$ verifies $	\hmax / \pa{h'^2 2^{p+1}\discount^{2h'}}
\geq  C\branchF(\nu,\rho)^{h'}\geq 1$. Therefore $p\leq\left\lfloor\log_2(\hmax/(h'^2\discount^{2h'}))\right\rfloor$.
\end{proof}

\section{Proof of Theorem~\ref{th:highnoise} and Theorem~\ref{th:lownoise}}
\label{app:proofALLmastro}
\restathhighnoise*
\restathlownoise*
%
\begin{proof}[Proof of~Theorem~\ref{th:highnoise} and Theorem~\ref{th:lownoise}]
	
	We first place ourselves on the event $\xi$ defined in Lemma~\ref{l:event} and where it is proven that $P(\xi)\geq 1 - \delta$. We bound the simple regret of \platypoos on $\xi$.

		\textbf{Step 1) General definition of the regret}
		
		We chose $\delta=\frac{4b(1-\delta)}{\Rmax\sqrt{\timeHorizon}}$ for the bound.
	We consider a problem with associated $(\nu,\rho)$. For simplicity we write $\branchF=\branchF^u(\nu,\rho)$.
	We have for all 
	$p\in[0:\pmax]$
	%
	\begin{align*}
		v(a^\timeHorizon)&+ \frac{b}{1-\discount^2}\sqrt{\frac{\pmax\log(\Rmax\timeHorizon^{3/2}/b)}{2\hmax}}\\
		&\geq u(a^\timeHorizon)+ \frac{b}{1-\discount^2}\sqrt{\frac{\pmax\log(\Rmax\timeHorizon^{3/2}/b)}{2\hmax}}       \stackrel{\textbf{(a)}}{\geq}
		\hat u(a^\timeHorizon)\\
	&	\stackrel{\textbf{(c)}}{\geq}
		\hat u(a^p)
		\stackrel{\textbf{(b)}}{\geq}
		\hat u\left(a^p_{[\depthOp_{\hmax,p}+1]}\right)\\
&		\stackrel{\textbf{(a)}}{\geq}
		u\left(a^p_{[\depthOp_{\hmax,p}+1]}\right) -\frac{b}{1-\discount^2}\sqrt{\frac{\pmax\log(\Rmax\timeHorizon^{3/2}/b)}{2\hmax}} \\ &		\stackrel{\textbf{(d)}}{\geq}
		v^\star- \nu\rho^{\depthOp_{\hmax,p}+1}  -\frac{b}{1-\discount^2}\sqrt{\frac{\pmax\log(\Rmax\timeHorizon^{3/2}/b)}{2\hmax}} 
	\end{align*}
	where \textbf{(a)} is because the 
	actions at time $t$, $a_t(\timeHorizon,p)$, of the \textit{candidate}  $a(\timeHorizon,p)$ have been evaluated $\frac{(t+1)\discount^{2t}\hmax}{(1-\discount)^2}$ times and because $\xi$ holds,  \textbf{(b)} is because $a^p_{[\depthOp_{\hmax,p}+1]}\in \{\action\in \actionSpace^\bullet:\forall t\in[2:\depth(\action)],\pullsNumber_{\action_{[t]}}\geq \left\lceil (t-1) 2^p\discount^{2(t-1)}\right\rceil\}$ and  $a^p = \argmax\limits_{\action\in \actionSpace^\bullet:\forall t\in[2:\depth(\action)],\pullsNumber_{\action_{[t]}}\geq \left\lceil (t-1) 2^p\discount^{2(t-1)}\right\rceil} \hat u(\action)$, \textbf{(c)} is because $a^\timeHorizon = \argmax\limits_{\{a^p ,p\in[0:\pmax]\}} \hat u(a^p)$, and 
	\textbf{(d)} is by Assumption~\ref{as:smoothu}.
	
	From the previous inequality we have 
	$r_\timeHorizon =  v^\star- Q^\star\left(x,a^\timeHorizon\right)\leq \nu\rho^{\depthOp_{\hmax,p}+1}+2\frac{b}{1-\discount^2}\sqrt{\frac{\pmax\log(\Rmax\timeHorizon^{3/2}/b)}{2\hmax}}$, for     $p\in[0:\pmax]$.

			\textbf{Step 2) Defining some important depths}
	For the rest of proof we want to lower bound $\max_{p\in[0:\pmax]}\depthOp_{\hmax,p}$. Lemma~\ref{lem:hstarSto} and ~\ref{lem:hstarSto2} provide some sufficient conditions  on $p$ and $h$ to get lower bounds. These conditions are inequalities in which as $p$ gets smaller (fewer samples) or $h$ gets larger (more depth) these conditions are more and more likely not to hold. For our bound on the regret of \platypoos to be small, we want quantities $p$ and $h$ where the inequalities hold but using as few samples as possible (small $p$) and having $h$ as large as possible. Therefore we are interested in determining when the inequalities flip signs which is when they turn to equalities. 
	This is what we solve next. 
	
	We set the notation $g_{n,b}^{\delta,\Rmax}=\pmax\log(\Rmax n^{3/2}/b(1-\delta))$.
	
	In its most general form we are interested in the real numbers
	$\tilde h$ and $\tilde p$ are such that $\tilde h$ is the larger real number such that for all $h\leq \tilde h'$
	
	\begin{equation} \label{eq:brokenbells}
	\frac{\hmax}{h^2 2^{\tilde p+1}\discount^{2h}}\geq  C\branchF(\nu,\rho)^{h}
	\text{ while } 
	\quad b\sqrt{ \frac{g_{n,b}^{\delta,\Rmax}}{2^{\tilde p}} } = \nu\rho^{\tilde h}
		\end{equation} 
	
	In the case $\discount^2\kappa\geq 1$ we can simply solve the following equations.	
	We denote $\tilde h_1$, $\tilde p_1$ the real numbers satisfying
	\begin{equation}\label{eq:eqStro}
	\frac{\hmax\nu^2\rho^{2\tilde h_1}}{2\tilde h^2_1b^2g_{n,b}^{\delta,\Rmax}\discount^{2\tilde h_1-2}}=C\kappa^{\tilde h_1}
	\quad \text{and} \quad b\sqrt{\frac{g_{n,b}^{\delta,\Rmax}}{2^{\tilde p_1}}} = \nu\rho^{\tilde h_1}\discount^2.  
	\end{equation}
	In the case $\discount^2\kappa\leq 1$  the previous equation can possess two solutions where the largest of these two solutions will not verify Equation~\ref{eq:brokenbells}. Additionally the smallest solution might be hard to express in a close form when  $\discount^2\kappa\leq 1$. Therefore for simplicity we define for the case $\discount^2\kappa\leq 1$,  $\tilde h_2$, $\tilde p_2$ the real numbers satisfying
		\begin{equation}\label{eq:eqStro2}
	\frac{\hmax\nu^2\rho^{2\tilde h_2}}{\tilde 2h^2_2b^2g_{n,b}^{\delta,\Rmax}}=C
	\quad \text{and} \quad b\sqrt{\frac{g_{n,b}^{\delta,\Rmax}}{2^{\tilde p_2}}} = \nu\rho^{\tilde h_2}.  
	\end{equation}
	$\tilde h_1$ and $\tilde p_1$ are defined for the case $\discount^2\kappa\geq 1$ while $\tilde h_2$ and $\tilde p_2$ are defined for the case $\discount^2\kappa\leq 1$.
	Our approach is to solve Equation~\ref{eq:eqStro} and~\ref{eq:eqStro2}  and then verify that it gives a valid indication of the behavior of our algorithm in term of its optimal $p$ and $h$.     We have  
	\[\tilde h_1 = \frac{2}{\log(\discount^2\kappa/\rho^2)}\lambertW\left(\log(\discount^2\kappa/\rho^2)/2\sqrt{\frac{\discount^2\nu^2\hmax }{2Cb^2g_{n,b}^{\delta,\Rmax}}}\right)\]
		\[\tilde h_2 = \frac{2}{\log(1/\rho^2)}\lambertW\left(\log(1/\rho^2)/2 \sqrt{\frac{\nu^2\hmax }{2Cb^2g_{n,b}^{\delta,\Rmax}}}\right)\]
	where standard $\lambertW$ is the Lambert $\lambertW$ function.  
	
	However after a close look at the Equation~\ref{eq:eqStro2}, we notice that it is possible to get values of $\tilde p$ and $\tilde h$   which would lead to a number of evaluations $\tilde h 2^{\tilde p}\discount^{\tilde h}<1$. This actually corresponds to an interesting case when the noise has a small range and where we can expect to obtain an improved result, that is: obtain a regret rate close to the deterministic case. This low range of noise case then has to be considered separately.
	
	Therefore, we distinguish two cases which corresponds to different noise regimes depending on the value of $b$. Looking at the equation on the right of~\eqref{eq:eqStro2}, we have that $\tilde h 2^{\tilde p}\discount^{\tilde h} < 1$ if $\frac{\nu^2\rho^{2\tilde h}}{\discount^{2\tilde h}\tilde h b^2g_{n,b}^{\delta,\Rmax}} > 1$. Based on this condition we now  consider the two cases. However for both of them we define some generic $\ddot h$ and     $\ddot p$.
	\paragraph{Case 1) $\discount^2\kappa\geq 1$ :}
	
	Note that in this case then $ \kappa> 1$.
	We subdivide this case into multiple subcases:
	
	\textbf{Case 1.1) Noise regime $\frac{\nu^2\rho^{2\tilde h_1}}{\discount^{2\tilde h_1}\tilde h_1 b^2g_{n,b}^{\delta,\Rmax}} \leq 1$}\\
		\textbf{Case 1.1.1) High-noise regime}
	$\frac{\nu^2\rho^{2\tilde h_1}}{ b^2g_{n,b}^{\delta,\Rmax}} \leq 1$	

	In this case, we denote $\ddot h_1 = \tilde h_1$ and $\ddot p_1 = \tilde p_1$.
	As $\frac{\nu^2\rho^{2\tilde h_1}}{ b^2g_{n,b}^{\delta,\Rmax}}	\leq 1$
	by construction,     we have $\tilde p_1\geq 0$.
	Using standard properties of the $\lfloor\cdot\rfloor$ function, 
	we have 
	\begin{equation}\label{eq:barbarSto2}
	b\sqrt{\frac{g_{n,b}^{\delta,\Rmax}}{2^{\left\lfloor\tilde p_1\right\rfloor+1}}}
	\leq
	b\sqrt{\frac{g_{n,b}^{\delta,\Rmax}}{2^{\tilde p_1}}} \leq
	\nu\rho^{\tilde h_1} \leq \nu\rho^{\left\lfloor\tilde h_1\right\rfloor}
	\end{equation}
	\begin{align*}\label{eq:barbarSto}
&	\text{    and,~~ }\frac{\hmax}{\left\lfloor\tilde h_1\right\rfloor  \left\lfloor\tilde h_1\right\rfloor 2^{\left\lfloor\tilde p_1\right\rfloor+1} \discount^{2 \left\lfloor\tilde h_1\right\rfloor } }  
	\geq
\frac{\hmax}{\left\lfloor\tilde h_1\right\rfloor  \left\lfloor\tilde h_1\right\rfloor 2^{\tilde p_1+1} \discount^{2 \left\lfloor\tilde h_1\right\rfloor } } \\
&	\geq
\frac{\hmax}{\left\lfloor\tilde h_1\right\rfloor  \tilde h_1 2^{\tilde p_1+1} \discount^{2 \tilde h_1 -2} } 
	=
	\frac{\hmax\nu^2\rho^{2\tilde h_1}}{2\tilde h^2_1b^2g_{n,b}^{\delta,\Rmax}\discount^{2\tilde h_1}}  \\
	&	= C\kappa^{ \tilde h_1}
	\geq C\kappa^{ \left\lfloor\tilde h_1\right\rfloor}.
	\end{align*}

We will verify that $\left\lfloor\ddot h\right\rfloor$ is a reachable depth by \platypoos in the sense that $\ddot h\leq \hmax$ and $\left\lfloor \ddot  p\right\rfloor\leq \left\lfloor\log_2(\hmax/(h^2\discount^{2h}))\right\rfloor$ and . 
	As $\kappa<1$, and $\ddot h\geq 0$ we have $\kappa^{\ddot h}\geq 1$. This gives  $C\kappa^{\ddot h}\geq 1$. Finally as 
	$ \frac{\hmax}{\ddot h^2  2^{\ddot p} \discount^{2\ddot h}}\geq C\kappa^{\ddot h}$, we have $\ddot h^2\discount^{2\ddot h}\leq \hmax/ 2^{\ddot p}$.
	
			\textbf{Case 1.1.2) Low-noise regime 1} 
			$\frac{\nu^2\rho^{2\tilde h_1}}{ b^2g_{n,b}^{\delta,\Rmax}} \geq 1$	
			

			 We denote  $\ddot h = \bar h_1$ and $\ddot p = \bar p_1$ where $\bar h$ and $\bar p$ verify,    
			\begin{equation}\label{eq:barbarSto564e}
			\frac{\hmax}{2\bar h^2_1\discount^{2\bar h_1}}=C\kappa^{\bar h_1}
			\quad \text{and} \quad \bar p_1 = 0.  
			\end{equation}
			Again, 	$\frac{\hmax}{2\bar h^2_12^{p_1}\discount^{2\bar h_1}}\geq 1$.
		
				\[\bar h_1 = \frac{2}{\log(\discount^2\kappa)}\lambertW\left(\log(\discount^2\kappa)/2\sqrt{\frac{\hmax }{2C}}\right)
				\]

				Using standard properties of the $\lfloor\cdot\rfloor$ function, 
				we have 
				\begin{equation}\label{eq:barbarSto57e}
				b\sqrt{\frac{g_{n,b}^{\delta,\Rmax}}{2^{\left\lfloor\ddot p_1\right\rfloor+1}}}
				\leq
				b\sqrt{g_{n,b}^{\delta,\Rmax}} <
				\nu\rho^{\tilde h_1} \stackrel{\textbf{(a)}}{\leq} \nu\rho^{\bar h_1} \leq \nu\rho^{\left\lfloor\bar h_1\right\rfloor}
				\end{equation}
				where \textbf{(a)} is because of the following reasoning.    As we have $	\frac{\hmax\nu^2\rho^{2\tilde h_1}}{2\tilde h^2_1b^2g_{n,b}^{\delta,\Rmax}\discount^{2\tilde h_1}}=C\kappa^{\tilde h_1}$ and $\frac{\nu^2\rho^{2\tilde h_1}}{ b^2g_{n,b}^{\delta,\Rmax}} \geq 1$, then, 
				$	\frac{\hmax}{2\tilde h^2_1\discount^{2\tilde h_1}}\leq C\kappa^{\tilde h_1}$. From the inequality $	\frac{\hmax}{2\tilde h^2_1}\leq C\kappa^{\tilde h_1}\discount^{2\tilde h_1}$ and the fact that $\bar h_1$ corresponds to the case of equality  $	\frac{\hmax}{2\bar h^2_1}=C\kappa^{\bar h_1}\discount^{2\bar h_1}$, we deduce that $\bar h_1\leq \tilde h_1$, since the left term of the inequality decreases with $h$ while the right term increases (as $\discount^2\kappa\geq 1$). Having $\bar h_1\leq \tilde h_1$ gives $\rho^{\bar h_1}\geq \rho^{\tilde h_1}$.
				
				Moreover,  the term $\log(\discount^2\kappa)$ of $\bar h_1$ could  lead to think that we could potentially obtain a better rate that in the deterministic case where the term is $\log(\discount^2\kappa)$. However this is not true because as $\bar h_1$ is the solution of 	$\bar h_1 = 	\frac{\hmax}{2\bar h^2_1\discount^{2\bar h_1}}=C\kappa^{\bar h_1}$ and we have by assumption in this case $\bar h_1 \discount{2\bar h_1 }\geq 1$ then $\bar h_1\leq h_3$  where $h_3$ is defined as the solution of $h_3 = 	\frac{\hmax}{2\bar h_3\discount^{2 h_3}}=C\kappa^{ h_3}$. We have  $h_3 =\frac{1}{\log(\kappa)}\lambertW\left(\log(\kappa)\frac{\hmax }{2C}\right)$. Therefore one can see that this rate is not better that the deterministic rates.
				
		\textbf{Case 1.2) Low noise regime 2 $\frac{\nu^2\rho^{2\tilde h_1}}{\discount^{2\tilde h_1}\tilde h_1 b^2g_{n,b}^{\delta,\Rmax}} \geq 1$}\\
	
	We denote  $\ddot h = \hat h_1$ and $\ddot p = \hat p_1$ where $\hat h$ and $\hat p$ verify,    
	\begin{equation}\label{eq:barbarSto564d}
	\frac{\hmax}{2\hat h_1}=C\kappa^{\hat h_1}
	\quad \text{and} \quad \hat p_1 = \max(0,\tilde p_1)).  
	\end{equation}
	
	\[\hat h_1 = \frac{1}{\log(\kappa)}\lambertW\left(\frac{\hmax \log(\kappa)}{2C}\right)\]
	
	Using standard properties of the $\lfloor\cdot\rfloor$ function, 
	we have 
	\begin{equation}\label{eq:barbarSto57d}
	b\sqrt{\frac{g_{n,b}^{\delta,\Rmax}}{2^{\left\lfloor\ddot p_1\right\rfloor+1}}}
	\leq
	b\sqrt{\frac{g_{n,b}^{\delta,\Rmax}}{2^{\tilde p_1}}} <
	\nu\rho^{\tilde h_1} \stackrel{\textbf{(a)}}{\leq} \nu\rho^{\hat h_1} \leq \nu\rho^{\left\lfloor\hat h_1\right\rfloor}
	\end{equation}
	where \textbf{(a)} is because of the following reasoning.    As we have $	\frac{\hmax\nu^2\rho^{2\tilde h_1}}{2\tilde h^2_1b^2g_{n,b}^{\delta,\Rmax}\discount^{2\tilde h_1}}=C\kappa^{\tilde h_1}$ and $\frac{\nu^2\rho^{2\tilde h_1}}{\discount^{2\tilde h_1}\tilde h_1 b^2g_{n,b}^{\delta,\Rmax}} \geq 1$, then, 
	$	\frac{\hmax}{2\tilde h_1} \leq C\kappa^{\tilde h_1}$. From the inequality 	$	\frac{\hmax}{2\tilde h_1} \leq C\kappa^{\tilde h_1}$ and the fact that $\hat h_1$ corresponds to the case of equality  $		\frac{\hmax}{2\hat h_1}=C\kappa^{\hat h_1}$, we deduce that $\hat h_1\leq \tilde h_1$, since the left term of the inequality decreases with $h$ while the right term increases . Having $\hat h_1\leq \tilde h_1$ gives $\rho^{\hat h_1}\geq \rho^{\tilde h_1}$.

	\textbf{Case 2) $\discount^2\kappa\leq 1$}\\
	\textbf{Case 2.1) Noise regime $\frac{\nu^2\rho^{2\tilde h}}{\discount^{2\tilde h}\tilde h b^2g_{n,b}^{\delta,\Rmax}} \leq 1$}\\
\textbf{Case 2.1.1) High-noise regime}
$\frac{\nu^2\rho^{2\tilde h_2}}{ b^2g_{n,b}^{\delta,\Rmax}} \leq 1$	

In this case, we denote $\ddot h = \tilde h_2$ and $\ddot p = \tilde p_2$.
As $\frac{\nu^2\rho^{2\tilde h_2}}{ b^2g_{n,b}^{\delta,\Rmax}}	\leq 1$
by construction,     we have $\tilde p_2\geq 0$.
Using standard properties of the $\lfloor\cdot\rfloor$ function, 
we have 
\begin{equation}\label{eq:barbarSto2z}
b\sqrt{\frac{g_{n,b}^{\delta,\Rmax}}{2^{\left\lfloor\tilde p_2\right\rfloor+1}}}
\leq
b\sqrt{\frac{g_{n,b}^{\delta,\Rmax}}{2^{\tilde p_2}}} =
\nu\rho^{\tilde h_2} \leq \nu\rho^{\left\lfloor\tilde h_2\right\rfloor}
\end{equation}
\begin{align*}\label{eq:barbarSto}
&	\text{    and,~~ }\frac{\hmax}{\left\lfloor\tilde h_2\right\rfloor  \left\lfloor\tilde h_2\right\rfloor 2^{\left\lfloor\tilde p_2\right\rfloor+1}  }  
\geq
\frac{\hmax}{\left\lfloor\tilde h_2\right\rfloor  \left\lfloor\tilde h_2\right\rfloor 2^{\tilde p_2+1}  } \\
&	\geq
\frac{\hmax}{\tilde h_2^2 2^{\tilde p_2+1}  } 
=
\frac{\hmax\nu^2\rho^{2\tilde h_2}}{2\tilde h^2_2b^2g_{n,b}^{\delta,\Rmax}}  \\
&	= C.
\end{align*}
We will verify that $\left\lfloor\ddot h\right\rfloor$ is a reachable depth by \platypoos in the sense that $\ddot h\leq \hmax$ and $\left\lfloor \ddot  p\right\rfloor\leq \left\lfloor\log_2(\hmax/(h^2\discount^{2h}))\right\rfloor$. 
As $\kappa<1$, and $\ddot h\geq 0$ we have $\kappa^{\ddot h}\geq 1$. This gives  $C\kappa^{\ddot h}\geq 1$. Finally as 
$ \frac{\hmax}{\ddot h^2  2^{\ddot p} \discount^{2\ddot h}}\geq C\kappa^{\ddot h}$, we have $\ddot h^2\discount^{2\ddot h}\leq \hmax/ 2^{\ddot p}$.

\textbf{Case 2.1.2) Low-noise regime 1} 		$\frac{\nu^2\rho^{2\tilde h_2}}{ b^2g_{n,b}^{\delta,\Rmax}} \geq 1$

We denote  $\ddot h = \bar h_2$ and $\ddot p = \bar p_2$ where $\bar h$ and $\bar p$ verify,    
\begin{equation}\label{eq:barbarSto564}
\frac{\hmax}{2\bar h^2_2}=C
\quad \text{and} \quad \bar p_2 = 0.  
\end{equation}
	Again, 	$\frac{\hmax}{2\bar h_2^2 2^{p_2}\discount^{2\bar h_2}}\geq 1$.

\[\bar h_2 = \sqrt{\frac{\hmax }{2C}}\]

Using standard properties of the $\lfloor\cdot\rfloor$ function, 
we have 
\begin{equation}\label{eq:barbarSto57}
b\sqrt{\frac{g_{n,b}^{\delta,\Rmax}}{2^{\left\lfloor\ddot p_2\right\rfloor+1}}}
\leq
b\sqrt{g_{n,b}^{\delta,\Rmax}} <
\nu\rho^{\tilde h_1} \stackrel{\textbf{(a)}}{\leq} \nu\rho^{\bar h_2} \leq \nu\rho^{\left\lfloor\bar h_2\right\rfloor}
\end{equation}
where \textbf{(a)} is because of the following reasoning.    As we have $\frac{\hmax\nu^2\rho^{2\tilde h_2}}{2\tilde h^2_2b^2g_{n,b}^{\delta,\Rmax}}  	= C$ and $\frac{\nu^2\rho^{2\tilde h_2}}{ b^2g_{n,b}^{\delta,\Rmax}} \geq 1$, then, 
$	\frac{\hmax}{2\tilde h^2_2}\leq C$. From the inequality $	\frac{\hmax}{2\tilde h^2_2}\leq C$ and the fact that $\bar h_2$ corresponds to the case of equality  $	\frac{\hmax}{2\bar h^2_2}=C$, we deduce that $\bar h_2\leq \tilde h_2$, since the left term of the inequality decreases with $h$ while the right term stays constant. Having $\bar h_2\leq \tilde h_2$ gives $\rho^{\bar h_2}\geq \rho^{\tilde h_2}$.

\textbf{Case 2.2) Low noise regime 2 $\frac{\nu^2\rho^{2\tilde h}}{\discount^{2\tilde h}\tilde h b^2g_{n,b}^{\delta,\Rmax}} \geq 1$}\\

We denote  $\ddot h = \hat h_2$ and $\ddot p = \hat p_2$ where $\hat h$ and $\hat p$ verify,    
\begin{equation}\label{eq:barbarSto5614}
\frac{\hmax}{\hat h_2}=C\kappa^{\hat h_2}
\end{equation}

\[\hat h_2 = \frac{1}{\log(\kappa)}\lambertW\left(\frac{\hmax \log(\kappa)}{C}\right)\]

By construction, we have $\tilde h_2\leq \tilde h$. 
We set
\begin{equation}\label{eq:barbarSto564cd}
 \hat p_2 = \max(0,\tilde p)).  
\end{equation}

Using standard properties of the $\lfloor\cdot\rfloor$ function, 
we have 
\begin{equation}\label{eq:barbarSto57j}
b\sqrt{\frac{g_{n,b}^{\delta,\Rmax}}{2^{\left\lfloor\ddot p_2\right\rfloor+1}}}
\leq
b\sqrt{\frac{g_{n,b}^{\delta,\Rmax}}{2^{\tilde p}}} =
\nu\rho^{\tilde h} \stackrel{\textbf{(a)}}{\leq} \nu\rho^{\hat h_2} \leq \nu\rho^{\left\lfloor\hat h_2\right\rfloor}
\end{equation}
where \textbf{(a)} is because of the following reasoning.    As we have $	\frac{\hmax\nu^2\rho^{2\tilde h}}{\tilde h^2b^2g_{n,b}^{\delta,\Rmax}\discount^{2\tilde h}}=C\kappa^{\tilde h}$
and $\frac{\nu^2\rho^{2\tilde h}}{\discount^{2\tilde h}\tilde h b^2g_{n,b}^{\delta,\Rmax}} \geq 1$, then, 
$	\frac{\hmax}{\tilde h} \leq \kappa^{\tilde h}$. From the inequality 	$	\frac{\hmax}{\tilde h} \leq C\kappa^{\tilde h}$ and the fact that $\hat h_2$ corresponds to the case of equality  $		\frac{\hmax}{2\hat h_2}=C\kappa^{\hat h_2}$, we deduce that $\hat h_2\leq \tilde h$, since the left term of the inequality decreases with $h$ while the right term increases. Having $\hat h_2\leq \tilde h$ gives $\rho^{\hat h_2}\geq \rho^{\tilde h}$.

\textbf{Step 3}
	Given these particular definitions of $\ddot h$ and $\ddot p$ in two distinct cases we now bound the regret.

	
	%


	We always have $\depthOp_{\hmax,\left\lfloor\ddot p\right\rfloor}\geq 0$. If $\ddot h \geq 1$, as discussed above  $\left\lfloor\ddot h \right\rfloor \in \left[\hmax\right]$,
	therefore $\depthOp_{\hmax,\left\lfloor\ddot p\right\rfloor}\geq \depthOp_{\left\lfloor\ddot h\right\rfloor,\left\lfloor\ddot p\right\rfloor},
	$ as $\depthOp_{\cdot,\left\lfloor p\right\rfloor}$ is increasing for all $p\in[0,\pmax ]$.
	Moreover on event~$\xi$, and for the cases 1.1.1, 1.1.2, 2.1.1 and 2.1.2 described above, 
	$ \depthOp_{\left\lfloor\ddot h\right\rfloor,\left\lfloor\ddot p\right\rfloor} =  \left\lfloor\ddot h\right\rfloor$ because of Lemma~\ref{lem:hstarSto2} (Case 2))  which assumptions on $\left\lfloor\ddot h\right\rfloor$ and $\left\lfloor\ddot p\right\rfloor$ are verified in each cases as detailed above and, in general,  $\left\lfloor\ddot h \right\rfloor \in \left[\left\lfloor\hmax/2^{\ddot p}\right\rfloor\right]$ and  $\left\lfloor\ddot p\right\rfloor\in [0: \pmax]$. So, for the aforementioned cases, we have $\depthOp_{\left\lfloor\hmax/2^{\ddot p}\right\rfloor,\left\lfloor\ddot p\right\rfloor}  \geq  \left\lfloor\ddot h \right\rfloor$.
	Very similarly cases 1.2 and 2.2. lead to $\depthOp_{\left\lfloor\hmax/2^{\ddot p}\right\rfloor,\left\lfloor\ddot p\right\rfloor}  \geq  \left\lfloor\ddot h \right\rfloor$ by using Lemma~\ref{lem:hstarSto2} (Case 1)).

	We bound the regret now discriminating on whether or not the event $\xi$ holds. We have 
	\begin{align*}
		r_\timeHorizon &\leq (1-\delta) \left(\nu\rho^{\depthOp_{\hmax,\ddot p}+1}+2\frac{b}{1-\discount^2}\sqrt{\frac{g_{n,b}^{\delta,\Rmax}}{2\hmax}}\right) 
		+\delta\times \frac{\Rmax}{1-\discount}\\
		&\leq
		\nu\rho^{\depthOp_{\hmax,\ddot p}+1}+2\frac{b}{1-\discount^2}\sqrt{\frac{g_{n,b}^{\delta,\Rmax}}{2\hmax}} +\frac{4b}{\sqrt{\timeHorizon}}\\
		&\leq
		\nu\rho^{\depthOp_{\hmax,\ddot p}+1}+6\frac{b}{1-\discount^2}\sqrt{\frac{g_{n,b}^{\delta,\Rmax}}{\hmax}}\cdot
	\end{align*}
	We can now bound the regret in the two regimes.
	\paragraph{Case 1) $\discount^2\kappa\geq 1$ :}

Note that in this case then $ \kappa> 1$.
We subdivide this case into multiple subcases:

\textbf{Case 1.1) Noise regime $\frac{\nu^2\rho^{2\tilde h_1}}{\discount^{2\tilde h_1}\tilde h_1 b^2g_{n,b}^{\delta,\Rmax}} \leq 1$}\\
\textbf{Case 1.1.1) High-noise regime}
	In general,we have
	\begin{align*}
		r_\timeHorizon
		&~\leq~ \nu\rho^{ 
			\frac{2}{\log(\discount^2\kappa/\rho^2)}\lambertW\left(\log(\discount^2\kappa/\rho^2)/2\sqrt{\frac{\discount^2\nu^2\hmax }{2Cb^2g_{n,b}^{\delta,\Rmax}}}\right)}+6\frac{b}{1-\discount}\sqrt{\frac{g_{n,b}^{\delta,\Rmax}}{\hmax}}\cdot
	\end{align*}
	
	Moreover, as proved by~\citet{Hoorfar08IO}, the Lambert $W(x)$ function verifies for $x\geq e$,
	$\lambertW(x)\geq \log\left(\frac{x}{\log x}\right)$.
	Therefore, if $\log(\discount^2\kappa/\rho^2)/2\sqrt{\frac{\discount^2\nu^2\hmax }{2Cb^2g_{n,b}^{\delta,\Rmax}}}>e$ we have,
	%
	\begin{align*}
	&	r_\timeHorizon-6\frac{b}{1-\discount}\sqrt{\frac{g_{n,b}^{\delta,\Rmax}}{\hmax}}\\
		&\leq 
		\nu \rho^{ \frac{2}{\log(\discount^2\kappa/\rho^2)}
			\left(
			\log\left(\frac{\log(\discount^2\kappa/\rho^2)/2\sqrt{\frac{\discount^2\nu^2\hmax }{2Cb^2g_{n,b}^{\delta,\Rmax}}}}{
				\log\left(\log(\discount^2\kappa/\rho^2)/2\sqrt{\frac{\discount^2\nu^2\hmax }{2Cb^2g_{n,b}^{\delta,\Rmax}}}    \right)
			}\right)
			\right)
		} \\
		&= \nu e^{ \frac{1}{\log(\discount^2\kappa/\rho^2)}
			\left(
	\log\left(\frac{\log^2(\discount^2\kappa/\rho^2)/2\frac{\discount^2\nu^2\hmax }{2Cb^2g_{n,b}^{\delta,\Rmax}}}{
	\log^2\left(\log(\discount^2\kappa/\rho^2)/2\sqrt{\frac{\discount^2\nu^2\hmax }{2Cb^2g_{n,b}^{\delta,\Rmax}}}    \right)
}\right)
			\right)
			\log(\rho)}\\
&		= \nu \left(\frac{\log^2(\discount^2\kappa/\rho^2)/2\frac{\discount^2\nu^2\hmax }{2Cb^2g_{n,b}^{\delta,\Rmax}}}{
	\log^2\left(\log(\discount^2\kappa/\rho^2)/2\sqrt{\frac{\discount^2\nu^2\hmax }{2Cb^2g_{n,b}^{\delta,\Rmax}}}    \right)
}\right)^{\frac{\log(\rho)}{\log(\discount^2\kappa/\rho^2)}}.
	\end{align*}
	%
	%

		\textbf{Case 1.1.2) Low-noise regime 1} 
$\frac{\nu^2\rho^{2\tilde h_1}}{ b^2g_{n,b}^{\delta,\Rmax}} \geq 1$	

	\begin{align*}
r_\timeHorizon
&~\leq~ \nu\rho^{\frac{2}{\log(\discount^2\kappa)}\lambertW\left(\log(\discount^2\kappa)/2\sqrt{\frac{\hmax }{2C}}\right)}+6\frac{b}{1-\discount}\sqrt{\frac{g_{n,b}^{\delta,\Rmax}}{\hmax}}\cdot
\end{align*}

Moreover, as proved by~\citet{Hoorfar08IO}, the Lambert $W(x)$ function verifies for $x\geq e$,
$\lambertW(x)\geq \log\left(\frac{x}{\log x}\right)$.
Therefore, if $\log(\discount^2\kappa)/2\sqrt{\frac{\hmax }{2C}}>e$ we have,
%
\begin{align*}
&	r_\timeHorizon-6\frac{b}{1-\discount}\sqrt{\frac{g_{n,b}^{\delta,\Rmax}}{\hmax}}\\
&\leq 
\nu \rho^{ \frac{2}{\log(\discount^2\kappa)}
	\left(
	\log\left(\frac{\log(\discount^2\kappa)/2\sqrt{\frac{\hmax }{2C}}}{
		\log\left(\log(\discount^2\kappa)/2\sqrt{\frac{\hmax }{2C}}    \right)
	}\right)
	\right)
} \\
&= \nu e^{ \frac{1}{\log(\discount^2\kappa)}
	\left(
	\log\left(\frac{\log^2(\discount^2\kappa)/2\frac{\hmax }{2C}}{
		\log^2\left(\log(\discount^2\kappa)/2\sqrt{\frac{\hmax }{2C}}   \right)
	}\right)
	\right)
	\log(\rho)}\\
&		= \nu \left(\frac{\log^2(\discount^2\kappa)/2\frac{\hmax }{2C}}{
	\log^2\left(\log(\discount^2\kappa)/2\sqrt{\frac{\hmax }{2C}}   \right)
}\right)^{\frac{\log(\rho)}{\log(\discount^2\kappa)}}.
\end{align*}
	We have $6b\sqrt{\frac{g_{n,b}^{\delta,\Rmax}}{\hmax}}\leq 
6\frac{\nu\rho^{\tilde h_1}}{\sqrt{g_{n,b}^{\delta,\Rmax}}} \sqrt{\frac{g_{n,b}^{\delta,\Rmax}}{\hmax}}\leq
6\nu\rho^{\tilde h_1}\leq
6\nu\rho^{\bar h_1}$.

Therefore
$r_\timeHorizon
\leq
\nu\rho^{\depthOp_{\hmax,\bar p}+1} +6b\sqrt{\frac{g_{n,b}^{\delta,\Rmax}}{\hmax}}
\leq
7\nu\rho^{\bar h_1} $.

	\textbf{Case 1.2) Low noise regime 2 $\frac{\nu^2\rho^{2\tilde h_1}}{\discount^{2\tilde h_1}\tilde h_1 b^2g_{n,b}^{\delta,\Rmax}} \geq 1$}\\

\begin{align*}
	r_\timeHorizon-6\frac{b}{1-\discount}\sqrt{\frac{g_{n,b}^{\delta,\Rmax}}{\hmax}}
\leq  \nu \left(\frac{\frac{\hmax \log(\kappa)}{2C}}{
	\log\left(\frac{\hmax \log(\kappa)}{2C}   \right)
}\right)^{\frac{\log(\rho)}{\log(\kappa)}}.
\end{align*}

Moreover if  $\frac{\nu^2\rho^{2\tilde h_1}}{ b^2g_{n,b}^{\delta,\Rmax}} \geq 1$, we have $6b\sqrt{\frac{g_{n,b}^{\delta,\Rmax}}{\hmax}}\leq 
6\frac{\nu\rho^{\tilde h_1}}{\sqrt{g_{n,b}^{\delta,\Rmax}}} \sqrt{\frac{g_{n,b}^{\delta,\Rmax}}{\hmax}}\leq
6\nu\rho^{\tilde h_1}\leq
6\nu\rho^{\hat h_1}$.

Therefore
$r_\timeHorizon
\leq
\nu\rho^{\depthOp_{\hmax,\bar p}+1} +6b\sqrt{\frac{g_{n,b}^{\delta,\Rmax}}{\hmax}}
\leq
7\nu\rho^{\hat h_1} $.

	\textbf{Case 2) $\discount^2\kappa\leq 1$}\\
	\textbf{Case 2.1) Noise regime $\frac{\nu^2\rho^{2\tilde h}}{\discount^{2\tilde h}\tilde h b^2g_{n,b}^{\delta,\Rmax}} \leq 1$}\\
	\textbf{Case 2.1.1) High-noise regime}
	$\frac{\nu^2\rho^{2\tilde h_2}}{ b^2g_{n,b}^{\delta,\Rmax}} \leq 1$

	\begin{align*}
		r_\timeHorizon-6\frac{b}{1-\discount}\sqrt{\frac{g_{n,b}^{\delta,\Rmax}}{\hmax}}
	\leq  \nu \left(\frac{\log^2(1/\rho^2)/2 \frac{\nu^2\hmax }{2Cb^2g_{n,b}^{\delta,\Rmax}}}{
		\log^2\left(\log(1/\rho^2)/2 \sqrt{\frac{\nu^2\hmax }{2Cb^2g_{n,b}^{\delta,\Rmax}}}  \right)
	}\right)^{-\frac{1}{2}}.
	\end{align*}

	\textbf{Case 2.1.2) Low-noise regime 1} 		$\frac{\nu^2\rho^{2\tilde h_2}}{ b^2g_{n,b}^{\delta,\Rmax}} \geq 1$
	
	Here with a similar reasoning as in the case 1.1.2) we have
	$r_\timeHorizon
	\leq
	7\nu\rho^{\bar h_1} \leq 7\nu \rho^{\sqrt{\frac{\hmax }{2C}}}$.

	\textbf{Case 2.2) Low noise regime 2 $\frac{\nu^2\rho^{2\tilde h}}{\discount^{2\tilde h}\tilde h b^2g_{n,b}^{\delta,\Rmax}} \geq 1$}\\

	\begin{align*}
		r_\timeHorizon-6\frac{b}{1-\discount}\sqrt{\frac{g_{n,b}^{\delta,\Rmax}}{\hmax}}
	\leq  \nu \left(\frac{\frac{\hmax \log(\kappa)}{C}}{
		\log\left(\frac{\hmax \log(\kappa)}{C}   \right)
	}\right)^{\frac{\log(\rho)}{\log(\kappa)}}.
	\end{align*}

	Moreover if  $\frac{\nu^2\rho^{2\tilde h}}{ b^2g_{n,b}^{\delta,\Rmax}} \geq 1$, we have $6b\sqrt{\frac{g_{n,b}^{\delta,\Rmax}}{\hmax}}\leq 
	6\frac{\nu\rho^{\tilde h}}{\sqrt{g_{n,b}^{\delta,\Rmax}}} \sqrt{\frac{g_{n,b}^{\delta,\Rmax}}{\hmax}}\leq
	6\nu\rho^{\tilde h}\leq
	6\nu\rho^{\hat h_2}$.
	
	Therefore
	$r_\timeHorizon
	\leq
	\nu\rho^{\depthOp_{\hmax,\bar p}+1} +6b\sqrt{\frac{g_{n,b}^{\delta,\Rmax}}{\hmax}}
	\leq
	7\nu\rho^{\hat h_1} $.

	Moreover if $\kappa=1$ then 	$r_\timeHorizon
	\leq  7\nu \rho^{\frac{\hmax }{C}}$
\end{proof}

\section{Use of the budget}
\begin{remark}
	The algorithm can be made anytime and agnostic to  $\timeHorizon$ using the standard doubling trick.
\end{remark}
\begin{remark}[More efficient use of the budget]\label{rem:effibud}
	Because of the use of the floor functions $\left\lfloor \cdot\right\rfloor$, the budget used in practice can be significantly smaller than $\timeHorizon$. While this only affects numerical constants in the bounds, in practice, it can noticeably influence the performance. Therefore one should consider, for instance, having $\hmax$ replaced by $c\times\hmax$ with $c$ been the largest number such that the budget is still smaller than $n$. Additionally,
	the use of the budget $\timeHorizon$ could be slightly optimized by taking into account that the necessary number of pulls at depth $h$ cannot be larger than $K^h$.
\end{remark}
\end{document}